\documentclass[11pt,letterpaper]{article}
\usepackage[margin=1in]{geometry}
\usepackage{amsmath}
\usepackage{amsfonts}
\usepackage{amsthm}
\usepackage{dsfont}
\usepackage[dvipsnames]{xcolor}
\usepackage{xspace}
\usepackage{amssymb}
\usepackage{hyperref}
\usepackage{graphicx}
\usepackage[algo2e,ruled,vlined]{algorithm2e}

\newtheorem{theorem}{Theorem}[section]
\newtheorem{lemma}{Lemma}[section]
\newtheorem{claim}{Claim}[section]
\newtheorem{definition}{Definition}[section]
\newtheorem{corollary}{Corollary}[section]

\newcommand{\bR}{\ensuremath{\mathbb R}}
\newcommand{\bE}{\ensuremath{\mathbb E}}

\newcommand{\cP}{\ensuremath{\mathcal P}}
\newcommand{\polylog}{\ensuremath{\mathrm{polylog}}}
\newcommand{\poly}{\ensuremath{\mathrm{poly}}}

\newcommand{\cC}{\ensuremath{\mathcal C}}
\newcommand{\cost}{\ensuremath{\mathrm{cost}}}
\newcommand{\one}{\ensuremath{\mathds 1}}

\SetKwFunction{tddecisiontree}{decision\_tree\_2d}
\SetKwFunction{tdsinglecut}{single\_cut\_2d}
\SetKwFunction{tdpostprocess}{post-process\_2d}
\SetKwFunction{decisiontree}{decision\_tree}
\SetKwFunction{singlecut}{single\_cut}
\SetKwFunction{postprocess}{post-process}

\DeclareMathOperator*{\argmax}{argmax}
\DeclareMathOperator*{\argmin}{argmin}

\definecolor{bluish-green}{HTML}{006E53}
\hypersetup{%
   colorlinks=true,%
   urlcolor=[rgb]{0.25,0.0,0.0},%
   linkcolor=[rgb]{0.5,0.0,0.2},%
   citecolor=bluish-green,%
   filecolor=[rgb]{0,0,0.4}, anchorcolor=[rgb]={0.0,0.1,0.2}%
}

\usepackage{cleveref}
\crefname{claim}{Claim}{Claims}
\title{Near-Optimal Explainable $k$-Means for All Dimensions}
\author{Moses Charikar\thanks{Computer Science Department, Stanford University. 
Email: \texttt{moses@cs.stanford.edu}. Supported by a Simons
Investigator Award.} \and Lunjia Hu\thanks{Computer Science Department, Stanford University. 
Email: \texttt{lunjia@stanford.edu}. Supported by NSF Award
IIS-1908774, the Simons Foundation collaboration on the theory of algorithmic fairness, and a VMware fellowship.}}
\date{}
\allowdisplaybreaks
\begin{document}
\maketitle
\begin{abstract}
\ifdefined\soda
\small\baselineskip=9pt
\fi
Many clustering algorithms are guided by certain cost functions such as the widely-used $k$-means cost. These algorithms divide data points into clusters with often complicated boundaries, creating difficulties in explaining the clustering decision. In a recent work, Dasgupta, Frost, Moshkovitz, and Rashtchian (ICML 2020) introduced explainable clustering, where the cluster boundaries are axis-parallel hyperplanes and the clustering is obtained by applying a decision tree to the data. The central question here is: how much does the explainability constraint increase the value of the cost function?

Given $d$-dimensional data points, we show an efficient algorithm that finds an explainable clustering whose $k$-means cost is at most $k^{1 - 2/d}\,\mathrm{poly}(d\log k)$ times the minimum cost achievable by a clustering without the explainability constraint, assuming $k,d\ge 2$. Taking the minimum of this bound and the $k\,\mathrm{polylog} (k)$ bound in independent work by Makarychev-Shan (ICML 2021), Gamlath-Jia-Polak-Svensson (2021), or Esfandiari-Mirrokni-Narayanan (2021), we get an improved bound of $k^{1 - 2/d}\,\mathrm{polylog}(k)$, which we show is optimal for every choice of $k,d\ge 2$ up to a poly-logarithmic factor in $k$. For $d = 2$ in particular, we show an $O(\log k\log\log k)$ bound, improving near-exponentially over the previous best bound of $O(k\log k)$ by Laber and Murtinho (ICML 2021).
\end{abstract}

\setcounter{page}{1}

\section{Introduction}
As a result of the rapid deployment of data analysis and machine learning techniques,
many important decision rules that impact our lives are 
learned from data by algorithms,
rather than
designed explicitly by people.
For unlabeled data,
clustering algorithms are a useful tool for learning such rules,
and a necessary step towards making the learned rules trustworthy is to
make them easily understood by people.
Many clustering algorithms are designed to optimize a cost function such as the widely-used $k$-means cost. While being a convenient and effective way for designing and analyzing clustering algorithms,
optimizing a simple cost function can produce clusterings that are not easily understood by people, causing interpretability issues when applied in practice \cite{saisubramanian2020balancing,MR4205290}.

We study a notion of explainable clustering introduced recently by
Dasgupta, Frost, Moshkovitz, and Rashtchian \cite{pmlr-v119-moshkovitz20a} that aims to improve clustering interpretability. Here, a clustering is considered explainable if 
the clusters are obtained by applying a decision tree with $k$ leaves to the data.
Specifically, every non-leaf node of the decision tree corresponds to an axis-parallel hyperplane that divides the current set of data points into two subsets, which are passed to the two children of the node respectively. Every leaf of the decision tree thus corresponds to all the data points contained in a (possibly unbounded) rectangular box with axis-parallel faces, and these data points are required to be placed in the same cluster.

To understand the increase in the clustering cost caused by the explainability constraint,
\cite{pmlr-v119-moshkovitz20a} 
studied the \emph{competitive ratio} of an explainable clustering, which is defined to be the ratio between its cost and the minimum cost achievable by a clustering using $k$ clusters without the explainability constraint.
For the $k$-medians and $k$-means cost,
\cite{pmlr-v119-moshkovitz20a} showed efficient algorithms computing explainable clusterings with competitive ratios $O(k)$ and $O(k^2)$ respectively. (In their work, distances are measured using the $\ell_1$-norm for $k$-medians, whereas the $\ell_2$-norm is used for $k$-means.)

The competitive ratio bounds in \cite{pmlr-v119-moshkovitz20a} have been improved significantly by several independent papers \cite{makarychev2021,gamlath2021nearly,esfandiari2021almost}. For explainable $k$-medians, \cite{makarychev2021} and \cite{esfandiari2021almost} achieved the current best competitive ratio $O(\log k\log\log k)$. For explainable $k$-means, the current best competitive ratio is $O(k\log k)$ by \cite{esfandiari2021almost}. 
These competitive ratio bounds are known to be near-optimal, with the current best lower bound being $\Omega(\log k)$ for $k$-medians \cite{pmlr-v119-moshkovitz20a,esfandiari2021almost} and $\Omega(k)$ for $k$-means \cite{gamlath2021nearly,esfandiari2021almost}.
 
The lower bounds above were proved only when the dimension of the data points is 
$d = \Omega(\log k)$.
This leaves open the question of achieving better competitive ratios %
for dimensions $d = o(\log k)$. 
Before our work, Laber and Murtinho \cite{laber2021price} gave algorithms with better competitive ratios in lower dimensions, but their bounds for $k$-medians and $k$-means were sub-optimal and subsumed by \cite{esfandiari2021almost}.
For $k$-medians, \cite{esfandiari2021almost} showed an $O(d(\log d)^2)$ competitive ratio in $d$ dimensions, which improves over the $O(\log k\log\log k)$ bound when $d \le O(\log k/\log\log k)$. 
For the popular $k$-means objective, no competitive ratio better than $O(k\log k)$ was known prior to our work even for $d = 2$.

In this work, we show that significantly better competitive ratios than $O(k\log k)$ can be achieved for $k$-means when the dimension 
$d = O(\log k/\log\log k)$.
We give a competitive ratio bound that depends on both $d$ and $k$, which we show is near-optimal for all choices of $d,k\ge 2$.

\paragraph{Our results.}
Our main result is an efficient algorithm that takes a set of $d$ dimensional points and computes an explainable clustering with competitive ratio at most $k^{1 - 2/d}\poly(d\log k)$ for the $k$-means cost (\Cref{thm:d>2}). 
Compared to the the previous best bound of $O(k\log k)$,
our bound has a better dependence on $k$ for \emph{every} fixed dimension $d$. 
The dependence on $d$
can be improved by taking the minimum of our bound and $O(k\log k)$. Specifically, if we run the algorithm in any of the independent work \cite{makarychev2021,gamlath2021nearly,esfandiari2021almost} instead when $d =  \Omega(\log k)$, we get an improved bound of $k^{1 - 2/d}\,\polylog (k)$ for all $k,d\ge 2$ (\Cref{cor:main}). We 
show this is near-optimal for all $k,d\ge 2$ by constructing a set of $d$ dimensional points for which the competitive ratio is at least $k^{1-2/d}/\polylog(k)$ for the $k$-means cost (\Cref{thm:lb}).

In the special case of $d = 2$, we show an efficient algorithm computing explainable clusterings with competitive ratio $O(\log k\log\log k)$ (\Cref{thm:2d}). This is a near-exponential improvement compared to the previous best bound of $O(k\log k)$ obtained first by \cite{laber2021price}.

\paragraph{Technical overview.} 
As previous work \cite{pmlr-v119-moshkovitz20a,laber2021price} and independent work \cite{makarychev2021,gamlath2021nearly,esfandiari2021almost}, we design a post-processing algorithm that takes an arbitrary clustering $\cC$ with $k$ clusters, and computes an explainable clustering with $k$-means cost at most $k^{1 - 2/d}\poly(d\log k)$ times the cost of $\cC$ using the same cluster centroids in $\cC$. 
Our competitive ratio bound is then achieved by running the post-processing algorithm on a clustering $\cC$ computed by a constant-factor approximation algorithm for $k$-means \cite{kanungo2004local,MR3734218,grandoni2021refined}.

We build the decision tree recursively starting from the root. That is, the first step of the algorithm is to find the hyperplane corresponding to the root of the decision tree, and then solve the two induced subproblems recursively. 
If a point $x$ and its assigned centroid in $\cC$ lie on different sides of the hyperplane, we need to re-assign a new centroid to $x$ that lies on the same side of the hyperplane with $x$. Every such re-assignment incurs some cost.
In a similar spirit to \cite{laber2021price}, we make sure that the re-assignment cost caused by the hyperplane is small, and that the two subproblems have similar ``sizes'' in order to minimize the depth of the decision tree. When the two goals are in conflict, it is important to make a balanced tradeoff: \cite{laber2021price} applies binary search with non-uniform probing cost due to \cite{MR1912303}, whereas we take a more flexible approach originating in an argument of Seymour \cite{MR1337358}. 

In previous analysis \cite{pmlr-v119-moshkovitz20a,laber2021price}, the re-assignment cost of every point is bounded above by the diameter of the current subproblem. 
While this ``diameter upper bound'' can be a good estimate when the dimension %
$d = \Omega(\log k)$,
we need a more careful bound to get a near optimal competitive ratio for %
smaller dimensions.
To this end,
we form a \emph{forbidden region} when selecting every hyperplane to prevent re-assigning a point if the re-assignment cost is too large compared to the current cost. 
We use a volume argument to bound the size of the forbidden region, so that we have enough non-forbidden space to apply Seymour's argument.
The strength of the volume argument increases significantly as the dimension $d$ decreases, resulting in our significantly improved bound in lower dimensions.

The ``diameter upper bound'' used in previous works allows them to completely ignore a point once it is re-assigned, because the cost of any further re-assignment can be covered by the current diameter. In our analysis, however, we need to deal with situations where a point is re-assigned multiple times.
We give every point a \emph{type} during the algorithm based on 
its ``re-assignment history'', and carefully control the cost at every re-assignment by designing the forbidden region based on the types of individual points.
In addition,
we show that ``essentially'' no point can be re-assigned too many times. Roughly speaking, for $d = 2$, we show that a point can be re-assigned at most twice before either a) the point has a large distance to its current centroid so that we can afford to use the diameter upper bound, or b) both the point and its current centroid are close to a corner of the current rectangle so that we can avoid re-assigning them further. For the more general case $d > 2$, it is possible that none of the above scenarios happen for some ``bad'' points, but we make sure that the ``bad'' points are assigned to only a small number of centroids, in which case we can also avoid re-assigning them further. 

\paragraph{Related work.}

Decision trees are a classic method for classifying labeled data \cite{hunt1966experiments}. 
Due to its intrinsic interpretability, people also applied decision trees and related algorithms to clustering unlabeled data: 
\cite{de1997using,chang2002new,basak2005interpretable,liu2005clustering,yasami2010novel,MR3066902,chen2016interpretable,ghattas2017clustering,bertsimas2018interpretable,MR4205290}. 
We use the framework of
\cite{pmlr-v119-moshkovitz20a}, who gave the first competitive ratio analysis for explainable clustering using decision trees. \cite{frost2020exkmc} relaxed the framework of \cite{pmlr-v119-moshkovitz20a} by allowing more leaves in the decision tree than the number of clusters, so that a cluster can correspond to multiple leaves.

In clustering tasks, the dimension of the input points plays an important role. Many influential results were obtained while studying clustering in different dimensions. While hardness results have been proved for approximately optimizing the $k$-medians and the $k$-means cost within a small constant factor \cite{MR2121521,MR3392820,bhattacharya2020hardness}, polynomial-time approximation schemes (PTAS) have been found for both $k$-medians and $k$-means in fixed dimensions \cite{MR1731569,MR1729137,MR3630998,MR3630997,MR3775817}.
A random projection to $O(\log k)$ dimensions approximately preserves the $k$-medians and the $k$-means cost of all $k$-clusterings of a given set of points with high probability \cite{MR4003407,MR4003406}.

Interpretability and explainability are important aspects of making machine learning reliable, and they have received growing research attention
(see \cite{molnar2020interpretable,murdoch2019definitions} for an overview).
Compared to clustering and unsupervised learning in general, more work on interpretability considered supervised learning \cite{ribeiro2016should,lundberg2017unified,adadi2018peeking,ribeiro2018anchors,lipton2018mythos,rudin2019stop,murdoch2019interpretable,alvarez2019weight,deutch2019constraints,sokol2020limetree,garreau2020explaining}. 
Besides decision trees,
neural nets have been used to improve clustering explainability
\cite{kauffmann2019clustering},
whereas an interpretability score was formulated by \cite{saisubramanian2020balancing}, who studied the tradeoff between the interpretability score and the clustering cost.
Fairness is another important consideration towards making clustering more trustworthy. There is a large body of recent work on fair clustering 
\cite{backurs2019scalable,NEURIPS2019_fc192b0c,NEURIPS2019_810dfbbe,kleindessner2019fair,schmidt2019fair,mahabadi2020individual,jung2020service,deeparnab2021better,vakilian2021improved}.

\paragraph{Independent work.}
Soon after we made this paper public on arXiv, three related and independent papers \cite{gamlath2021nearly,esfandiari2021almost,makarychev2021} appeared on arXiv, and later \cite{makarychev2021} also appeared in the proceedings of ICML 2021. As we mentioned earlier, all three papers achieved similar competitive ratios for explainable $k$-medians with the $\ell_1$-norm and explainable $k$-means with the $\ell_2$-norm that are near-optimal when the dimension 
$d = \Omega(\log k)$.
In addition, \cite{gamlath2021nearly} considered general $\ell_p$-norms and achieved competitive ratio $O(k^{p-1}(\log k)^2)$ when the distances are raised to the $p$-th power with a near-matching lower bound, \cite{esfandiari2021almost} achieved a dimension-dependent competitive ratio $O(d(\log d)^2)$ for $k$-medians with the $\ell_1$-norm in $d$ dimensions, and \cite{makarychev2021} achieved a competitive ratio $O((\log k)^{3/2})$ for $k$-medians with the $\ell_2$-norm.

\paragraph{Paper organization.}
We formally define explainable clustering and introduce relevant notation in \Cref{sec:preli}. We prove our $O(\log k\log\log k)$ competitive ratio upper bound for $d = 2$ in \Cref{sec:2d}, and our $k^{1 - 2/d}\,\poly(d\log k)$ upper bound for $d > 2$ in \Cref{sec:d>2}. The lower bound $k^{1 - 2/d}/\polylog(k)$ is shown in \Cref{sec:lb}. Some helper claims and lemmas used in our analysis are stated and proved in \Cref{sec:helper}.

\section{Preliminaries}
\label{sec:preli}
We use $x(j)$ to denote the $j$-th coordinate of a point $x$ in the $d$-dimensional space $\bR^d$, where $j$ is chosen from $[d]$, namely, $\{1,\ldots,d\}$. We use $\|x\|_2 = (\sum_{j=1}^d x(j)^2)^{1/2}$ and $\|x\|_\infty = \max_{j\in [d]}|x(j)|$ to denote the $\ell_2$-norm and the $\ell_\infty$-norm of a point $x\in \bR^d$, respectively.

Every pair $(j,\theta)\in [d]\times \bR$ defines an axis-parallel hyperplane that partitions $\bR^d$ into two subsets: $B_\le(j,\theta):=\{x\in\bR^d:x(j) \le \theta\}$ and $B_>(j,\theta):=\{x\in\bR^d:x(j) > \theta\}$. We say two points $x,x'\in\bR^d$ lie on the same side of the hyperplane $(j,\theta)$ if $x,x'\in B_\le(j,\theta)$ or $x,x'\in B_>(j,\theta)$; otherwise we say the two points lie on different sides of the hyperplane, or equivalently, they are separated by the hyperplane.

We consider decision trees as rooted directed trees. Nodes in the tree with no child are called \emph{leaves}, and we require that every non-leaf node has exactly $2$ children\textemdash a left child and a right child. Every non-leaf node corresponds to an axis-parallel hyperplane $(j,\theta)\in [d]\times \bR$. This naturally makes every node in the tree define a subset of $\bR^d$ with axis-parallel boundaries:
the root defines the entire space $\bR^d$;
if a non-leaf node corresponding to hyperplane $(j,\theta)$ defines the region $B$,
its left child defines the region $B\cap B_\le (j,\theta)$, and its right child defines the region $B\cap B_>(j,\theta)$. Clearly, the regions defined by the leaves of a decision tree form a partition of $\bR^d$.

For positive integers $d$ and $k$,
a $k$-clustering $\cC$ for a set of points $x_1,\ldots,x_n\in\bR^d$ consists of
$k$ centroids $y_1,\ldots,y_k\in\bR^d$ and an assignment mapping $\xi:[n]\rightarrow [k]$. 
The $k$-means cost of the clustering $\cC$ is given by $\cost(\cC) = \sum_{i=1}^n\|x_i - y_{\xi(i)}\|_2^2$. 
We say the clustering $\cC$ is $k$-explainable with respect to
a decision tree $T$ if $T$ has at most $k$ leaves  and $\xi(i) = \xi(i')$ holds for all points $x_i,x_{i'}$ in the region defined by the same leaf of $T$.

For a subset $W\subseteq \bR^{d}$, we use $|W|$ to denote its Lebesgue measure. We only care about the Lebesgue measure of bounded subsets that can be represented as a union of finitely many rectangles (or intervals when $d = 1$). For those subsets $W$, the Lebesgue measure $|W|$ always exists.

We use $\log$ to denote the base-$e$ logarithm, and $\log_2$ to denote the base-$2$ logarithm.

\section{Explainable $k$-means in the plane}
\label{sec:2d}
We focus on the simpler case $d = 2$ in this section and give an efficient algorithm for finding a $k$-explainable clustering with competitive ratio $O(\log k\log\log k)$.
Before we describe our algorithm, we remark that there exists a poly-time algorithm that computes a $k$-explainable clustering with \emph{minimum} $k$-means cost given a set of $n$ input points in $d = 2$ dimensions. In fact, for general $d\ge 2$ and $n\ge 2$, there exists such an algorithm with running time $n^{O(d)}$ via dynamic programming: 
if $k\ge n$, it is trivial to achieve zero cost; 
if $k < n$, the algorithm solves all subproblems each consisting of a box $(\alpha(1),\beta(1)]\times (\alpha(2),\beta(2)]\times \cdots \times (\alpha(d),\beta(d)]\subseteq \bR^d$ and a positive integer $k'\le k$, where the goal is to find a $k'$-explainable clustering with minimum cost for the input points inside the box. Although there are infinitely many such boxes, at most $n^{O(d)}$ among them define distinct subsets of input points, so essentially there are at most $kn^{O(d)} = n^{O(d)}$ different subproblems. Also, every subproblem with $k'=1$ can be solved directly in $O(nd)$ time, and every subproblem with $k'>1$ can be solved in $O(k'nd) = (nd)^{O(1)}$ time using solutions to subproblems with smaller $k'$.

While the above algorithm guarantees to find a $k$-explainable clustering with minimum cost and thus minimum competitive ratio, it does not give us a concrete bound on the competitive ratio. We develop a different algorithm that post-processes an arbitrary $k$-clustering $\cC$ into a $k$-explainable clustering, and we show that the $k$-means cost of the explainable clustering is at most $O(\log k\log\log k)$ times the cost of $\cC$ assuming $d = 2$. Choosing $\cC$ as the output of a constant-factor approximation algorithm for $k$-means ensures that the explainable clustering has competitive ratio $O(\log k\log\log k)$. 

\begin{theorem}
\label{thm:2d}
Assume $k \ge 2$. 
There exists a poly-time algorithm $\tdpostprocess$ that takes a $k$-clustering $\cC$ of $n$ points in $2$ dimensions, and outputs a clustering $\cC'$ of the $n$ points and a decision tree $T$ with at most $k$ leaves such that
\begin{enumerate}
\item $\cC'$ is $k$-explainable with respect to $T$;
\item $\cost(\cC') \le O(\log k\log\log_2(2k))\cdot \cost(\cC)$;
\item $\cC'$ uses the same $k$ centroids as $\cC$ does.
\end{enumerate}
Consequently, there exists a poly-time algorithm that takes $n$ points in $2$ dimensions and outputs a $k$-explainable clustering with competitive ratio $O(\log k\log\log_2(2k))$.
\end{theorem}

In the rest of the section, we assume $d = 2$ and $k\ge 2$.

\subsection{Subproblem}
\label{sec:2d-subproblem}
Our algorithm $\tdpostprocess$ works in a recursive manner, constructing the tree from root to leaf. Thus, in each stage of the algorithm, we focus on a subset of the points and the centroids. Moreover, our algorithm keeps track of some helper information for every point. This leads us to the definition of a \emph{subproblem}.
\begin{definition}[Subproblem for $d = 2$] 
\label{def:2d-subproblem}
Given points $x_1,\ldots,x_n\in\bR^d$ and centroids $y_1,\ldots,y_k\in\bR^d$,
a subproblem $\cP$ consists of the following:
\begin{enumerate}
\item A subset $X\subseteq\{x_1,\ldots,x_n\}$. We focus on points $x\in X$.
\item A subset $Y\subseteq\{y_1,\ldots,y_k\}$. We focus on centroids $y\in Y$.
\item An assigned centroid $\sigma_x\in Y$ for every $x\in X$.
\item \label{item:L-infty} A length $\ell_x\ge 0$ for every point $x\in X$. We always enforce $\ell_x$ to be an upper bound on $\|x - \sigma_x\|_\infty$ (see \Cref{def:2d-valid} \Cref{def:2d-valid-1}).
While we define the $k$-means cost using the $\ell_2$-norm, in our analysis we find it more convenient to keep track of the $\ell_\infty$-norm instead.

\item A type $t_x$ for every point $x\in X$. The type $t_x$ is either a function $t_x:[d]\rightarrow \{0,1,2\}$ or the \emph{irrelevant type} $t_x = \bot$. This gives a partition of $X$ into two subsets: a subset 
\[
R = \{x\in X:t_x\ne \bot\}
\]
consisting of all \emph{relevant points}, and a subset $X\backslash R$ consisting of all \emph{irrelevant points}. The set $R$ is further partitioned into $R_0,R_1, \ldots, R_d$ defined as follows:
\[
R_i = \{x\in R:\|t_x\|_0 = i\}, \quad \textnormal{for all}\ i\in[d],
\]
where $\|t_x\|_0$ denotes the number of $j\in [d]$ with $t_x(j)\ne 0$. 
If $x\in R$ has $t_x(j)\ne 0$ for some $j\in [d]$, we ensure that $x$ is close to one of the \emph{boundaries} in the $j$-th dimension, which we formalize in \Cref{def:boundaries} and \Cref{def:2d-valid} \Cref{def:2d-valid-4}. \end{enumerate}
\end{definition}
We fix a positive real number $m$ as the \emph{centroid mass} which we determine later. We can now define two quantities $M(\cP)$ and $A(\cP)$ for a subproblem $\cP$ with respect to the centroid mass $m$. 

\begin{definition}
\label{def:cost}
Given a subproblem $\cP$, we define the following quantities:
\begin{align*}
M(\cP) & := m|Y| + \sum_{x\in R_0} \ell_x^2,\\
A(\cP) & := f(M(\cP)) + 2^{32} \sum_{x\in R_1} \ell_x^2 + 2^9\sum_{x\in R_2} \ell_x^2 + \sum_{x\in X\backslash R} \ell_x^2,
\end{align*}
where
\[
f(M) := 2^{57}M(1 + \log (M/m))\log\log_2 (2k).
\]
\end{definition}
The quantity $A(\cP)$ plays an important role in our proof of \Cref{thm:2d}. The algorithm $\tdpostprocess$ we construct for proving \Cref{thm:2d} forms an initial subproblem $\widetilde \cP$ with $A(\widetilde \cP) \le O(\log k\log\log_2 (2k))\cdot \cost(\cC)$ and divides it into smaller and smaller subproblems. We compute an explainable clustering for every subproblem, and we choose $\cC'$ to be the explainable clustering for the initial subproblem $\widetilde\cP$. 
We show that the cost of the explainable clustering we find for every subproblem $\cP$ is at most $2A(\cP)$ (\Cref{lm:2d-induction}), which implies $\cost(\cC') \le 2A(\widetilde \cP)\le O(\log k\log\log_2 (2k))\cdot \cost(\cC)$, as required by \Cref{thm:2d}.

Our algorithm $\tdpostprocess$ crucially uses the boundaries and the diameter of a subproblem defined as follows:
\begin{definition}[Subproblem boundary]
\label{def:boundaries}
Given a subproblem $\cP$, for every $j\in[d]$, we define
\begin{align*}
b_1(j) & = \min_{y\in Y}y(j),\ \textnormal{and}\\
b_2(j) & = \max_{y\in Y}y(j)
\end{align*}
as the lower and upper \emph{boundaries} in the $j$-th dimension. 
Define $L:=\max_{j\in[d]}(b_2(j) - b_1(j))$ as the diameter of the subproblem $\cP$.
\end{definition}
To impose necessary constraints on the subproblems we deal with, we focus on \emph{valid} subproblems defined below (see \Cref{fig:1} for an example of a valid subproblem).
\begin{definition}[Valid subproblem]
\label{def:2d-valid}
Given points $x_1,\ldots,x_n\in\bR^d$ and centroids $y_1,\ldots,y_k\in\bR^d$,
a subproblem $\cP = (X, Y, (\sigma_x)_{x\in X}, (\ell_x)_{x\in X}, (t_x)_{x\in X})$ is valid if all of the following hold:
\begin{enumerate}
\item \label{def:2d-valid-1} $\ell_x \ge \|x - \sigma_x\|_\infty$ for all $x\in X$.
\item \label{def:2d-valid-2} $M(\cP)/m \le 2k$.
\item \label{def:2d-valid-3} For every point $x\in X\backslash R$ and every $y\in Y$, $\ell_x \ge \|x - y\|_\infty$.
\item \label{def:2d-valid-4} If point $x\in R$ has $t_x(j) = 1$ for some $j\in[d]$, then $|x(j) - b_1(j)|\le \ell_x$.
Similarly,
if $x\in R$ has $t_x(j) = 2$, then $|x(j) - b_2(j)|\le \ell_x$.
\end{enumerate}
\end{definition}
\begin{figure}[h]
\centering
\includegraphics[width=0.8\textwidth]{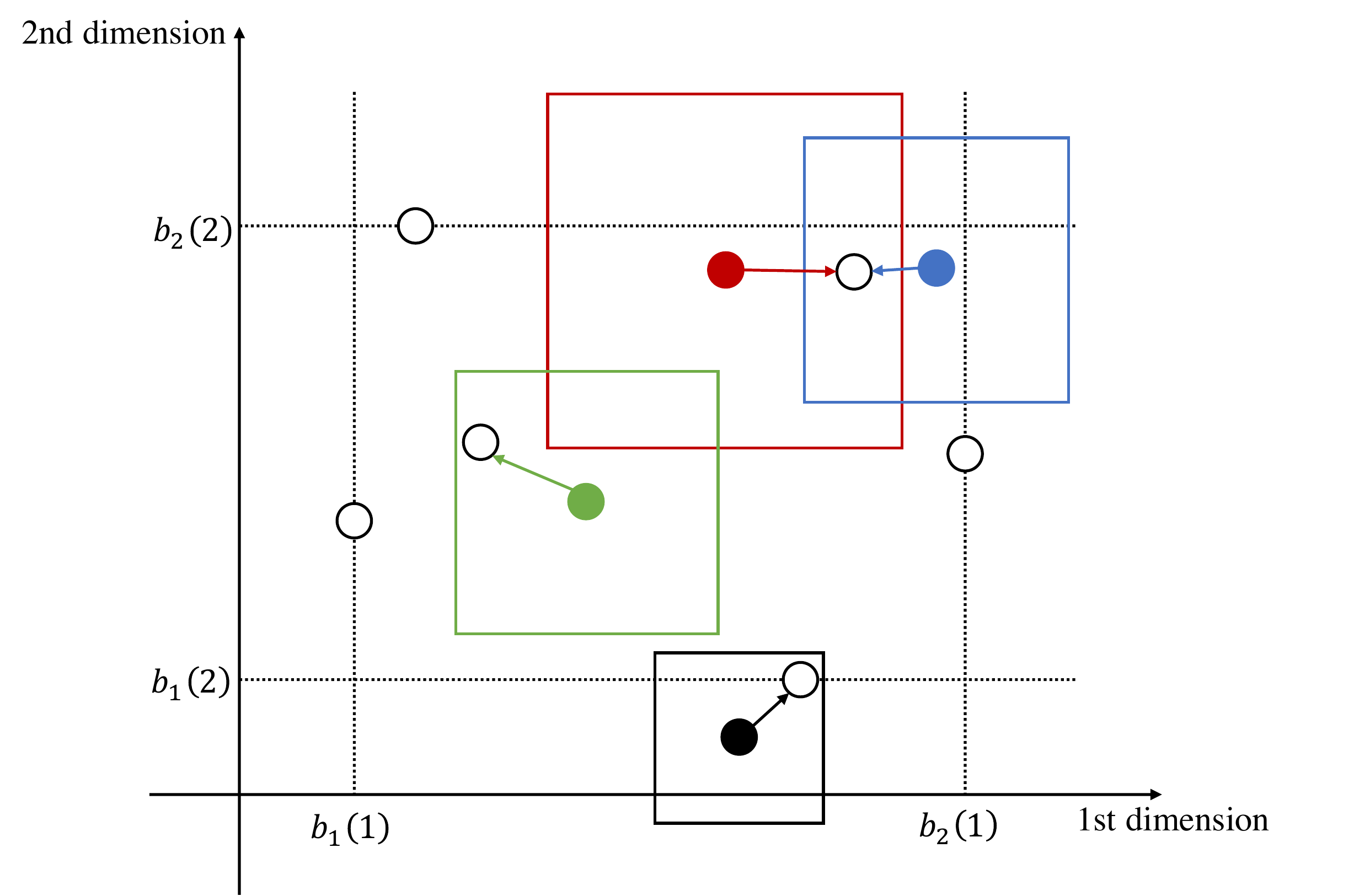}
\caption{A valid subproblem in $2$ dimenions. The hollow circles represent the centroids $y\in Y$, and the filled circles represent the relevant points $x\in R$. Irrelevant points $x\in X\backslash R$ are not shown in the figure. 
The four dotted lines represent the boundaries as in \Cref{def:boundaries}.
Every point $x\in R$ has an arrow pointing to its assigned centroid $\sigma_x$.   
We draw an $\ell_\infty$ ball (square) around every point $x\in R$ with radius $\ell_x$. The square around a point $x$ contains $\sigma_x$ by \Cref{def:2d-valid} \Cref{def:2d-valid-1}. According to \Cref{def:2d-valid} \Cref{def:2d-valid-4}, the blue point can have $t_x(1) = 2$ and $t_x(2) = 2$ because the blue square intersects the upper boundaries in both dimensions. Similarly, the red point can have $t_x(2) = 2$ but must have $t_x(1) = 0$. The green point must have $t_x(1) = t_x(2) = 0$. The black point to the bottom lies below the lower boundary in the $2$nd dimension, and it can have $t_x(2) = 1$.}
\label{fig:1}
\end{figure}

\subsection{Making a single cut}
\label{sec:2d-single-cut}
We describe an efficient algorithm $\tdsinglecut$ that takes a valid subproblem $\cP$, and produces two smaller valid subproblems $\cP_1$ and $\cP_2$ together with an axis-parallel hyperplane $(j^*,\theta)$ that separates them. Later in \Cref{sec:2d-tree}, we invoke this algorithm recursively to construct the algorithm $\tdpostprocess$ required by \Cref{thm:2d}.

Specifically, given an input subproblem $\cP = (X, Y, (\sigma_x)_{x\in X}, (\ell_x)_{x\in X}, (t_x)_{x\in X})$, the algorithm $\tdsinglecut$ computes a partition $X_1,X_2$ of $X$, a partition $Y_1,Y_2$ of $Y$, new assignments $(\sigma'_x)_{x\in X}$, new lengths $(\ell_x')_{x\in X}$, new types $(t_x')_{x\in X}$, and outputs two smaller subproblems
\begin{align}
\cP_1 & = (X_1, Y_1, (\sigma_x')_{x\in X_1}, (\ell_x')_{x\in X_1}, (t_x')_{x\in X_1}),\notag\\
\cP_2 & = (X_2, Y_2, (\sigma_x')_{x\in X_2}, (\ell_x')_{x\in X_2}, (t_x')_{x\in X_2}).
\label{eq:subproblems}
\end{align}

The partitions $(X_1,X_2)$ and $(Y_1,Y_2)$ are determined by an axis-parallel hyperplane $(j^*,\theta)\in[d]\times (b_1(j^*), b_2(j^*))$:
\begin{align}
X_1 & = X\cap B_\le (j^*,\theta) = \{x\in X:x(j^*) \le \theta\},\notag\\
X_2 & = X\cap B_> (j^*,\theta)= \{x\in X:x(j^*) > \theta\},\notag\\
Y_1 & = Y\cap B_\le (j^*,\theta)= \{y\in Y:y(j^*) \le \theta\},\notag\\
Y_2 & = Y\cap B_> (j^*,\theta)= \{y\in Y:y(j^*) > \theta\}. \label{eq:partition}
\end{align}
We always choose $j^*\in[d]$ so that $b_2(j^*) - b_1(j^*)$ is maximized, i.e., 
$b_2(j^*) - b_1(j^*) = L$.
Note that by choosing $\theta\in(b_1(j^*),b_2(j^*))$, we are implicitly requiring $b_1(j^*)<b_2(j^*)$, or equivalently, $L > 0$, which we assume to be the case. Moreover, the choice $\theta\in(b_1(j^*),b_2(j^*))$ guarantees that $Y_1$ and $Y_2$ are both non-empty, and thus they both have sizes smaller than $|Y|$, which means that the two subproblems are indeed ``smaller''.

We say a point $x\in X$ is \emph{$\sigma$-separated} if $x$ and $\sigma_x$ are separated by the hyperplane $(j^*,\theta)$. 
In other words, $\sigma$-separated points form the subset $X_+\subseteq X$ defined as follows:
\begin{equation}
\label{eq:separated}
X_+ := \{x\in X_1:\sigma_x\in Y_2\} \cup \{x\in X_2:\sigma_x \in Y_1\}.
\end{equation}
Consequently, non-$\sigma$-separated points belong to one of the following two sets
\begin{align}
X_{11} & :=\{x\in X_1:\sigma_x\in Y_1\}, \quad \textnormal{and}\notag\\
X_{22} & := \{x\in X_2:\sigma_x\in Y_2\}.
\label{eq:non-separated}
\end{align}

To make sure that the two subproblems $\cP_1,\cP_2$ are well-defined, we require that $\sigma'_x\in Y_1$ whenever $x\in X_1$ and $\sigma'_x\in Y_2$ whenever $x\in X_2$. In other words, we require that no point is $\sigma'$-separated.
This implies that for every $\sigma$-separated point $x\in X$, we must ensure $\sigma'_x\ne \sigma_x$.
On the other hand, our algorithm $\tdsinglecut$ guarantees $\sigma'_x = \sigma_x$ whenever $x$ is not $\sigma$-separated.

Our goal is to show that the two new subproblems $\cP_1,\cP_2$ created by the $\tdsinglecut$ algorithm satisfy the following lemmas, which are crucial in our analysis to obtain \Cref{thm:2d}. We always assume that the input subproblem $\cP$ is valid throughout \Cref{sec:2d-single-cut} even when we do not explicitly state so.
\begin{lemma}
\label{lm:2d-valid}
The two new subproblems $\cP_1,\cP_2$ output by $\tdsinglecut$ are both valid.
\end{lemma}
\begin{lemma}
\label{lm:2d-cost}
The two new subproblems $\cP_1,\cP_2$ output by $\tdsinglecut$ satisfy $A(\cP_1) + A(\cP_2) \le A(\cP)$.
\end{lemma}
We prove the above lemmas after describing the $\tdsinglecut$ algorithm step by step in the following subsections. 
\subsubsection{Preprocessing}
\label{sec:2d-preprocessing}
For every $x\in R$ with $\ell_x\ge L/16$, we have 
\begin{equation}
\label{eq:preprocessing}
\|x - y\|_\infty\le \|x - \sigma_x\|_\infty + \|\sigma_x - y\|_\infty \le \ell_x + L \le 17\ell_x, \quad \textnormal{for all}\ y\in Y.
\end{equation}
For every such point $x$, we replace the current value of $\ell_x$ by $17\ell_x$, and set $t_x = \bot$ (thus removing $x$ from $R$). The new subproblem is still valid (\Cref{def:2d-valid} \Cref{def:2d-valid-3} follows from \eqref{eq:preprocessing}), and all points $x\in R$ in the new subproblem satisfies $\ell_x\le L/16$. Moreover, it is clear that the value of $A(\cP)$ does not increase (note that $17^2 < 2^9$). For the rest of \Cref{sec:2d-single-cut}, we use $\cP = (X, Y, (\sigma_x)_{x\in X}, (\ell_x)_{x\in X}, (t_x)_{x\in X})$ to denote the subproblem \emph{after} the preprocessing step.

\subsubsection{Forbidding}
\label{sec:2d-forbidding}
We specify a subset $F$ of the interval $(b_1(j^*),b_2(j^*))$ as the forbidden region. By making the algorithm $\tdsinglecut$ choose $\theta$ outside of the forbidden region, we can guarantee some desired properties for $\sigma$-separated points (see \Cref{lm:2d-separated-relevant}).

For every point $x\in X$, define 
$Y(x)\subseteq Y$ as the following set of centroids:
\begin{equation}
Y(x):=\left\{\begin{array}{ll}
\{y\in Y:y(j^*) \ge \min\{x(j^*), b_2(j^*)\}\},& \textnormal{if}\ x(j^*) \ge \sigma_x(j^*);
\\ 
\{y\in Y:y(j^*) \le \max\{x(j^*), b_1(j^*)\}\},& \textnormal{if}\ x(j^*) < \sigma_x(j^*).
\end{array}\right.
\label{eq:candidate-target}
\end{equation}
If $x$ is $\sigma$-separated, it is clear that every centroid in $Y(x)$ must lie on the same side of the hyperplane $(j^*,\theta)$ with $x$, in which case choosing $\sigma_x'$ from $Y(x)$ prevents $x$ from being $\sigma'$-separated.
Define $\eta_x$ as the centroid in $Y(x)$ with the smallest $\ell_\infty$ distance to $\sigma_x$, and define $q_x$ to be the corresponding $\ell_\infty$ distance:
\begin{align}
\eta_x & = \argmin_{y\in Y(x)}\|\sigma_x - y\|_\infty,\notag\\
q_x & =\|\sigma_x - \eta_x\|_\infty.
\label{eq:target}
\end{align}
Now we focus on points $x$ in the subset $T\subseteq S \subseteq R_0\cup R_1$ defined as follows:
\begin{align}
S & =\{x\in R_0\cup R_1:t_x(j^*) = 0, x(j^*)\ne \sigma_x(j^*)\}.\notag\\
T & = \left\{x\in S: \frac{q_x}{L}>2^{11}\left(\frac{\ell_x}{L}\right)^{1/(d - \|t_x\|_0)}\right\}.
\label{eq:T}
\end{align}
For every point $x\in X$, we define an interval $W_x$ as follows:
\begin{equation}
\label{eq:Wx}
W_x = \left\{
\begin{array}{ll}
[\sigma_x(j^*), x(j^*)) \cap (b_1(j^*), b_2(j^*)), & \textnormal{if}\ x(j^*) \ge \sigma_x(j^*);
\\ 
\,\![x(j^*), \sigma_x(j^*)) \cap (b_1(j^*), b_2(j^*)), & \textnormal{if}\ x(j^*) < \sigma_x(j^*).
\end{array}
\right.
\end{equation}
We define the forbidden region $F$ as
\[
F = (b_1(j^*),b_1(j^*) + L/8]\cup[b_2(j^*) - L/8, b_2(j^*))\cup \bigcup_{x\in T}W_x.
\]
It is clear that $F$ can be represented as a union of finitely many \emph{disjoint} intervals, and the representation can be computed in poly-time.
We define $F$ as above because choosing $\theta$ outside $F$ guarantees that all relevant points that can possibly be $\sigma$-separated must have good properties summarized in \Cref{lm:2d-type} and \Cref{lm:2d-separated-relevant} below, including having a relatively small value of $q_x$.
Since the algorithm needs to choose $\theta\in(b_1(j^*),b_2(j^*))$ outside the forbidden region, it is necessary to show that the forbidden region does not cover the entire interval $(b_1(j^*),b_2(j^*))$. \Cref{lm:forbidden} below makes a stronger guarantee. 

\begin{lemma}
\label{lm:2d-type}
If we choose $\theta\in (b_1(j^*),b_2(j^*))\backslash F$, then every $x\in R$ with $t_x(j^*) = 1$ belongs to $X_{11}$, and similarly every $x\in R$ with $t_x(j^*) = 2$ belongs to $X_{22}$. Consequently, every $x\in R\cap X_+$ satisfies $t_x(j^*) = 0$.
\end{lemma}
\begin{proof}
Consider a point $x\in R$ with $t_x(j^*) = 1$.
By \Cref{def:2d-valid} \Cref{def:2d-valid-4}, we know $|x(j^*) - b_1(j^*)| \le \ell_x \le L/16$, where the last inequality is guaranteed by the preprocessing step. By \Cref{def:2d-valid} \Cref{def:2d-valid-1} and the triangle inequality, $|\sigma_x(j^*) - b_1(j^*)| \le \|x - \sigma_x\|_\infty + |x(j^*) - b_1(j^*)| \le \ell_x + \ell_x \le L/8$. Therefore, 
\[
\max\{x(j^*), \sigma_x(j^*)\}\le b_1(j^*) + L/8 < \theta,
\]
where the last inequality is because $(b_1(j^*), b_1(j^*) + L/8]\subseteq F$ and $\theta\in(b_1(j^*),b_2(j^*))\backslash F$. This implies $x\in X_{11}$. Similarly, every $x\in R$ with $t_x(j^*) = 2$ belongs to $X_{22}$. Since $X_+$ is disjoint from $X_{11}\cup X_{22}$, points $x\in R\cap X_+$ cannot have $t_x(j^*)\in \{1,2\}$, and thus $t_x(j^*) = 0$.
\end{proof}
\begin{lemma}
\label{lm:2d-separated-relevant}
If we choose $\theta\in (b_1(j^*),b_2(j^*))\backslash F$, then every $\sigma$-separated relevant point $x\in R\cap X_+$ satisfies all of the following:
\begin{enumerate}
\item \label{lm:2d-separated-relevant-1} $x(j^*)\ne \sigma_x(j^*)$;
\item \label{lm:2d-separated-relevant-2} $t_x(j^*) = 0$ (and thus $\|t_x\|_0 \le 1$ and $x\in R_0\cup R_1$);
\item \label{lm:2d-separated-relevant-3} $\frac{q_x}{L}\le 2^{11}(\frac{\ell_x}{L})^{1/(d - \|t_x\|_0)}$.
\end{enumerate}
\end{lemma}
\begin{proof}

\Cref{lm:2d-separated-relevant-1} is obvious, since a point $x$ cannot be $\sigma$-separated unless $x(j^*)\ne \sigma_x(j^*)$. 
\Cref{lm:2d-separated-relevant-2} follows directly from \Cref{lm:2d-type}. 

Assume for the sake of contradiction that \Cref{lm:2d-separated-relevant-3} is not satisfied by $x\in R\cap X_+$. We already know that \Cref{lm:2d-separated-relevant-1} and \Cref{lm:2d-separated-relevant-2} are both satisfied, so $x\in S$, and therefore $x\in T$. This implies $W_x\subseteq F$ by the definition of $F$. However, the fact that $x\in X_+$ implies $\theta\in W_x$, and thus $\theta\in F$, a contradiction.
\end{proof}

\begin{lemma}
\label{lm:forbidden}
The forbidden region $F$ has length (i.e.\ Lebesgue measure) at most $L/2$.
\end{lemma}
\Cref{lm:forbidden} is a direct consequence of the following lemma:
\begin{lemma}
\label{lm:volume}
$|\bigcup_{x\in T}W_x|\le L/4$.
\end{lemma}
\begin{proof}
While we are dealing with $d = 2$ specifically, we prove the lemma using a more general language so that the proof can be reused in \Cref{sec:d>2} where we deal with $d > 2$.

By \Cref{claim:disjoint-intervals}, we can find $U\subseteq T$ such that the intervals $(W_x)_{x\in U}$ are disjoint, and $|\bigcup_{x\in T}W_x|\le 3|\bigcup_{x\in U}W_x|$.

It remains to prove that $|\bigcup_{x\in U}W_x|\le L/12$. 
Define $U_> = \{x\in U:x(j^*) > \sigma_x(j^*)\}$.

We prove $|\bigcup_{x\in U_>}W_x|\le L/24$ via a volume argument. For every point $x\in U_>$, define a rectangular box $B_x\subseteq \bR^d$ as follows:
\[
B_x = \{z\in\bR^d: \|z - \sigma_x\|_\infty < q_x/3, z(j^*) - \sigma_x(j^*) > q_x/6  \}.
\]
We can write $B_x$ in a different way as the cartesian product $B_x = B_x(1)\times \cdots \times B_x(d)$, where $B_x(j)\subseteq \bR$ is the interval $(\sigma_x(j) - q_x/3, \sigma_x(j) + q_x/3)$ if $j\ne j^*$, and $B_x(j^*)\subseteq \bR$ is the interval $(\sigma_x(j^*) + q_x/6, \sigma_x(j^*) + q_x/3)$. Thus, the width of $B_x$ in the $j^*$-th dimension is $1/4$ times the width in other dimensions.

We show that the boxes $B_x$ are pair-wise disjoint. Assume for the sake of contradiction that a point $z$ lies in both boxes $B_x$ and $B_{x'}$ where $x,x'$ are distinct points in $U_>$. Assume w.l.o.g.\  $x(j^*) \le x'(j^*)$. Since $W_x$ and $W_{x'}$ are disjoint, we have $\sigma_x(j^*) < x(j^*) \le \sigma_{x'}(j^*) < x'(j^*)$. Therefore, $\sigma_{x'}\in Y(x)$, and thus
\[
q_x \le \|\sigma_x - \sigma_{x'}\|_\infty \le \|\sigma_x - z\|_\infty + \|\sigma_{x'} - z\|_\infty < q_x/3 + q_{x'}/3.
\]
This implies that $q_x < q_{x'}/2$, and thus $z(j^*) < \sigma_x(j^*) + q_x/3 < \sigma_{x'}(j^*) + q_{x'}/6 < z(j^*)$, a contradiction.

It is clear by definition that $q_x \le L$ for all $x\in U_>$. Therefore, the boxes $B_x$ are all contained in the large box
\[
B:= \{z\in\bR^d: \forall j\in[d],b_1(j) - L/3 \le z(j) \le b_2(j) + L/3\}.
\]
The large box $B$ contains a smaller box $B'$ defined as follows:
\[
B':= \{z\in B: \forall j\in [d]\backslash\{j^*\}, b_1(j) \le z(j) \le b_2(j)\}.
\]
For every $z\in B$, we define $\pi_z\in B'$ such that
$\pi_z(j^*) = z(j^*)$ and for all $j\ne j^*$, 
\[
\pi_z(j) = \left\{\begin{array}{ll} 
z(j),& \textnormal{if}\ b_1(j) \le z(j) \le b_2(j),
\\
b_1(j) , & \textnormal{if}\ z(j) < b_1(j),
\\
b_2(j), & \textnormal{if}\ z(j) > b_2(j).
\end{array}\right.
\]
This allows us to define another family of disjoint boxes. 
Specifically, define $\widetilde B_x = \widetilde B_x(1)\times \cdots \times \widetilde B_x(d)$ where $\widetilde B_x(j)\subseteq \bR$ are defined as follows:
\begin{enumerate}
\item $\widetilde B_x(j^*)= (\sigma_x(j^*) + q_x/6, \sigma_x(j^*) + q_x/3)$;
\item for all $j\ne j^*$ with $t_x(j) = 0$, $\widetilde B_x(j) = (\sigma_x(j) - q_x/3, \sigma_x(j) + q_x/3)$;
\item for all $j$ with $t_x(j) = 1$, $\widetilde B_x(j) = (b_1(j) - L/3, b_1(j))$;
\item for all $j$ with $t_x(j) = 2$, $\widetilde B_x(j) =  (b_2(j), b_2(j) + L/3)$.
\end{enumerate}
It is clear that $\widetilde B_x\subseteq B$. Moreover, for $j\in [d]$ with $t_x(j) = 1$, we have 
\begin{equation}
\label{eq:forbid-project}
0 \le  \sigma_x(j) - b_1(j)\le |b_1(j) - x(j)| + \|x - \sigma_x\|_\infty \le 2\ell_x < q_x/3,
\end{equation}
where the last inequality is by $x\in U_>\subseteq T$ and thus $q_x > 2^{11}\ell_x$.
Inequality \eqref{eq:forbid-project} implies $b_1(j) \in B_x(j)$ whenever $t_x(j) = 1$.
Similarly, we have $b_2(j) \in B_x(j)$ whenever $t_x(j) = 2$.
Therefore, any $z\in \widetilde B_x$ satisfies $\pi_z\in B_x$ because $z(j)\in \widetilde B_x(j)$ implies $\pi_z(j)\in B_x(j)$ for all $j\in [d]$. It is then easy to show that $\widetilde B_x$ are disjoint: if $z\in \widetilde B_x\cap\widetilde B_{x'}$, then $\pi_z\in B_x\cap B_{x'}$, a contradiction. \Cref{fig:2} shows an example of the boxes $B_x$ and $\widetilde B_x$ contained in $B$.
\begin{figure}[h]
\centering
\includegraphics[width=0.8\textwidth]{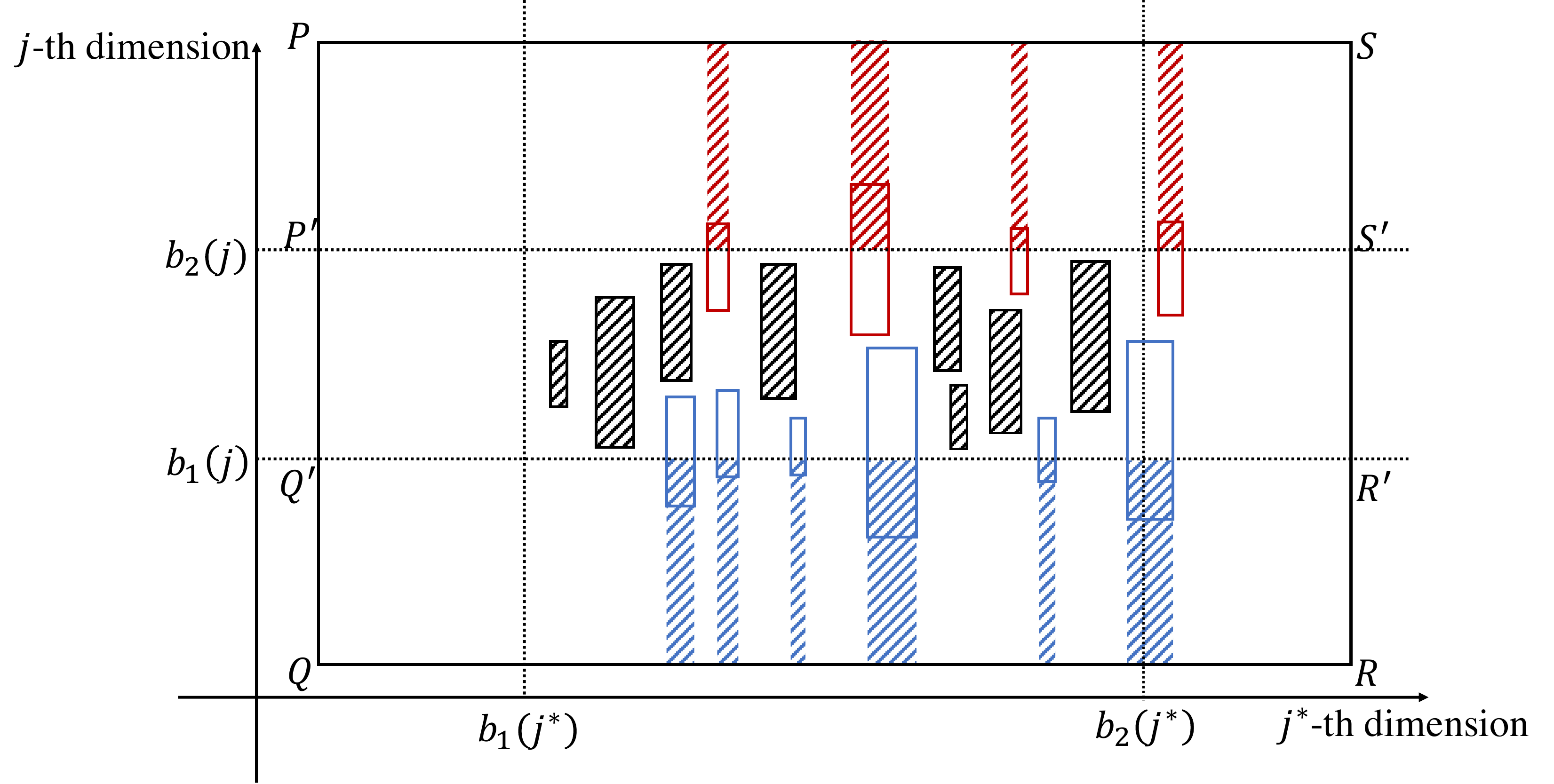}
\caption{For every $x\in U_>$, the box $B_x$ is represented by a square with solid boundary, and $\widetilde B_x$ is represented by a shaded square. The rectangle $PQRS$ represents the box $B$, and the rectangle $P'Q'R'S'$ represents the smaller box $B'$. 
If $t_x(j) = 1$, the corresponding $B_x$ (blue) intersects the lower boundary in the $j$-th dimension (the dotted horizontal line at $b_1(j)$). 
If $t_x(j) = 2$, the corresponding $B_x$ (red) intersects the upper boundary in the $j$-th dimension (the dotted horizontal line at $b_2(j)$). }
\label{fig:2}
\end{figure}

The volume of $\widetilde B_x$ can be lower bounded as follows:
\[
|\widetilde B_x| = \prod_{j=1}^d|\widetilde B_x(j)| = (1/4)(2q_x/3)^{d - \|t_x\|_0}(L/3)^{\|t_x\|_0}\ge (5L/3)^d(24\ell_x/L),
\]
where the last inequality is by the fact that $\frac{q_x}{L}> 2^{11}(\frac{\ell_x}{L})^{1/(d - \|t_x\|_0)}$, $d\ge 2$, and $\|t_x\|_0 \le 1$. Summing up, we have
\[
\sum_{x\in U_>}(5L/3)^d(24\ell_x/L) \le \sum_{x\in U_>}|\widetilde B_x| \le |B|\le  (5L/3)^d.
\]
Therefore,
\[
\sum_{x\in U_>}|W_x| \le \sum_{x\in U_>}\ell_x \le L/24.
\]

A similar argument proves $\bigcup_{x\in U\backslash U_>}|W_x|\le L/24$, which implies $|\bigcup_{x\in U}W_x|\le L/12$ and completes the proof of the lemma.
\end{proof}
\subsubsection{Cutting}
\label{sec:2d-cutting}
Our algorithm $\tdsinglecut$ chooses $\theta\in(b_1(j^*),b_2(j^*))\backslash F$ using a method by Seymour \cite{MR1337358} based on \Cref{lm:2d-cutting} below. Recall that any choice of $\theta$ defines a partition $X_1,X_2$ of $X$ and a partition $Y_1,Y_2$ of $Y$ as specified in \eqref{eq:partition}. 
It also defines $X_+, X_{11}, X_{22}$ as specified in \eqref{eq:separated} and \eqref{eq:non-separated}. We further define
\begin{align*}
M_1^* & = m|Y_1| + \sum_{x\in R_0\cap X_{11}}\ell_x^2,\\
M_2^* & = m|Y_2| + \sum_{x\in R_0\cap X_{22}}\ell_x^2,\\
M^* &= \min\{M(\cP)/2, M(\cP) - M_1^*, M(\cP) - M_2^*\}.
\end{align*}
\begin{lemma}
\label{lm:2d-cutting}
There exists $\theta \in (b_1(j^*), b_2(j^*))\backslash F$ satisfying
\begin{equation}
\label{eq:2d-cutting}
\sum_{x\in R_0\cap X_+}\ell_xL\le 8M^*\log (M(\cP)/M^*)\log\log_2(M(\cP)/m).
\end{equation}
Moreover, $\theta$ can be computed in poly-time.
\end{lemma}
\begin{proof}
The fact that $\theta$ can be computed in poly-time follows immediately from its existence, because there are at most $|X| + |Y|$ choices of $\theta\in (b_1(j^*),b_2(j^*))\backslash F$ that lead to distinct partitions $(X_1,X_2)$ and $(Y_1,Y_2)$. It only takes poly-time to check \eqref{eq:2d-cutting} for each of the choices using the representation of $F$ as a union of disjoint intervals. Below we prove the existence of $\theta$.

Define $M = M(\cP)$.
For every point $x\in R_0$, define a function $g_x:(b_1(j^*),b_2(j^*))\rightarrow [0,+\infty)$ such that 
$g_x(\theta) = \ell_x^2/|W_x|$ if $\theta\in W_x$, and
$g_x(\theta) = 0$ otherwise. 
Define $h(\theta)$ as the number of centroids $y\in Y$ with $y(j^*)< \theta$.
Define 
\[
G(\theta) = mh(\theta) + \sum_{x\in R_0} \int_{b_1(j^*)}^\theta  g_x(\theta')\mathrm d\theta'.
\]
It is clear that $G$ is non-decreasing, and bounded between $m$ and $M - m$ for all $\theta \in(b_1(j^*),b_2(j^*))$. Moreover, for every choice of $\theta$, we have $M_1^* \le G(\theta)$ and $M_2^* \le M - G(\theta)$. Define $I_1 = (\{x(j^*):x\in X\}\cup\{y(j^*):y\in Y\})\cap(b_1(\ell^*), b_2(\ell^*))$. $G$ is differentiable on $(b_1(\ell^*),b_2(\ell^*))\backslash I_1$, where $G'(\theta) = \sum_{x\in R_0} g(\theta)$.

By \Cref{lm:forbidden}, the total length of the non-forbidden region is at least $L/2$. Therefore, we can find real numbers $\alpha_1,\ldots, \alpha_u$ and $\beta_1,\ldots,\beta_u$ such that 
\begin{enumerate}
\item $b_1(j^*)\le \alpha_1 < \beta_1 \le \alpha_2 < \beta_2 \le \cdots \le \alpha_u < \beta_u \le b_2(j^*)$;
\item every $(\alpha_i,\beta_i)$ is disjoint from the forbidden region $F$;
\item $\sum_{i=1}^u(\beta_i - \alpha_i) = L/2$.
\end{enumerate}
Define $z_i:=\sum_{i' = 1}^{i}(\beta_{i'} - \alpha_{i'})$. 
We define a bijection $\gamma$ from $\bigcup_{i=1}^u(z_{i-1},z_i)$ to $\bigcup_{i=1}^u(\alpha_i,\beta_i)$ as follows: for all $z\in (z_{i-1},z_i)$, define $\gamma(z) = \alpha_i + (z - z_{i-1})$. It is clear that $\gamma$ is non-decreasing and has derivative $\gamma'(z) = 1$ for all $z\in \bigcup_{i=1}^u(z_{i-1},z_i)$.

Define $I = \{z_1,\ldots,z_{m - 1}\}\cup \{\gamma^{-1}(\theta):\theta\in I_1\cap \bigcup_{i=1}^u(\alpha_i,\beta_i)\}$. $I$ is a finite subset of $(0,L/2)$.
Define $V:(0,L/2)\backslash I\rightarrow [m,M - m]$ by $V(z) = G(\gamma(z))$. 
Then $V$ is a non-decreasing function on $(0,L/2)\backslash I$ with derivative $V'(z) = G'(\gamma(z))$.
By \Cref{lm:mean-value-2}, we can find $z\in (0, L/2)\backslash I$ such that 
\[
V'(z) \le (4/L)M'\log (M/M')\log\log_2(M/m), 
\]
where $M':= \min\{V(z), M - V(z)\}$. Choose $\theta =\gamma (z)$. 
We have $M'\le V(z) = G(\theta) \le M - M_2^*$ and $M'\le M - V(z) = M - G(\theta) \le M - M_1^*$.
Therefore, $M' \le M^*$.
By \Cref{claim:monotone-relaxed},
we have
\[
G'(\theta)  = V'(z) \le (8/L)M^*\log (M/M^*)\log\log_2(M/m).
\]
The lemma is proved by noting that 
\[
G'(\theta) = \sum_{x\in R_0} g_x(\theta) \ge \sum_{x\in R_0\cap X_+}g_x(\theta) \ge \sum_{x\in R_0\cap X_+}\ell_x,
\]
where the last inequality is by the easy fact that $\theta\in W_x$ whenever $x\in X_+$ and that $\ell_x^2/|W_x| \ge \ell_x$.
\end{proof}
\subsubsection{Updating}
\label{sec:2d-updating}
Having computed the hyperplane $(j^*,\theta)$, we get the partions $(X_1,X_2)$ and $(Y_1,Y_2)$ by \eqref{eq:partition}.
We now specify the new assignments $\sigma'_x$, new lengths $\ell_x'$, new types $t_x'$. The two new subproblems $\cP_1,\cP_2$ can then be formed by \eqref{eq:subproblems}.

For every non-$\sigma$-separated point $x\in X\backslash X_+$ we define $\sigma'_x = \sigma_x$, $\ell_x' = \ell_x$, and $t_x' = t_x$. For every $\sigma$-separated irrelevant point $x\in X_+\backslash R$, we define $\ell_x' = \ell_x$, $t_x' = t_x(=\bot)$, and 
define $\sigma'_x$ to be an arbitrary centroid in $Y$ that lies on the same side of the hyperplane $(j^*,\theta)$ with $x$. Such a centroid exists because $Y_1,Y_2$ are both non-empty since we choose $\theta$ from $(b_1(j^*),b_2(j^*))$.

It remains to consider relevant points that are $\sigma$-separated, i.e.\ points $x\in R\cap X_+$. These points satisfy the properties in \Cref{lm:2d-separated-relevant}. For these points, we define 
\begin{equation}
\label{rule:2d-ell}
\ell_x' = \ell_x + 2^{11}L(\ell_x/L)^{1/(d - \|t_x\|_0)} 
\end{equation}
and $\sigma'_x = \eta_x$. We define $t_x'$ to be equal to $t_x$, except that we change $t_x'(j^*)$ to either $1$ or $2$ from the original value $t_x(j^*) = 0$ (\Cref{lm:2d-separated-relevant} \Cref{lm:2d-separated-relevant-2}). Specifically, define $t_x'(j^*) = 1$ if $x\in X_2$, and $t_x'(j^*) = 2$ if $x\in X_1$.

This completes our definition of $\sigma_x', \ell_x'$ and $t_x'$.
The algorithm $\tdsinglecut$ returns the two subproblem $\cP_1,\cP_2$ formed by \eqref{eq:subproblems} together with the hyperplane $(j^*,\theta)$.
Before we prove \Cref{lm:2d-valid} and \Cref{lm:2d-cost}, we first prove \Cref{lm:2d-M-star} below.
Define $R' = \{x\in X:t'_x\ne \bot\}$.
For $i = \{0,1,2\}$, define $R'_i = \{x\in R':\|t'_x\|_0 = i\}$.
It is clear from our update rules that $t_x' = \bot$ if and only if $t_x = \bot$, so $R' = R$.
\begin{lemma}
\label{lm:2d-M-star}
We have the following equalities and inequalities:
\begin{align}
M(\cP_1) = M_1^*, & \quad \textnormal{and} \quad M(\cP_2) = M_2^*; \label{eq:2d-M-star-1}\\
\min\{M_1^*,M_2^*\} \le M^*, & \quad \textnormal{and} \quad \max\{M_1^*,M_2^*\} \le M(\cP) - M^*. \label{eq:2d-M-star-2}
\end{align}
\end{lemma}
\begin{proof}
According to our update rule, no point $x\in X_+$ has $\|t_x'\|_0 = 0$ because either $x\in X_+\backslash R$ and $t_x' = \bot$, or $x\in R\cap X_+$ and $\|t_x'\|_0 \ge 1$. Therefore,
a point $x\in X_1\cap R_0'$ if and only if $x\in X_1\backslash X_+ = X_{11}$, $t_x\ne \bot$, and $\|t_x\|_0 = 0$, or equivalently, $x\in R_0\cap X_{11}$.
This implies
\[
M(\cP_1) = m|Y_1| + \sum_{x\in X_1\cap R'_0}(\ell_x')^2 = m|Y_1| + \sum_{x\in R_0\cap X_{11}}(\ell_x')^2 = m|Y_1| + \sum_{x\in R_0\cap X_{11}}\ell_x^2 = M_1^*.
\]
Similarly, $M(\cP_2) = M_2^*$. This completes the proof of \eqref{eq:2d-M-star-1}.

To prove \eqref{eq:2d-M-star-2}, we assume w.l.o.g.\ that $M_1^* \le M_2^*$. It is clear from definition that $M_1^* + M_2^* \le M(\cP)$. Therefore, $M_1^* \le \min\{M(\cP)/2, M(\cP) - M_1^*, M(\cP) - M_2^*\} = M^*$. The definition of $M^*$ directly implies $M_2^* \le M(\cP) - M^*$.
\end{proof}
We conclude \Cref{sec:2d-single-cut} by proving \Cref{lm:2d-valid} and \Cref{lm:2d-cost}.
\ifdefined\soda
We first prove \Cref{lm:2d-valid}.
\begin{proof}
\else
\begin{proof}[Proof of \Cref{lm:2d-valid}]
\fi
Let $\cP = (X, Y, (\sigma_x)_{x\in X}, (\ell_x)_{x\in X}, (t_x)_{x\in X})$ denote the valid subproblem \emph{after} the preprocessing step.

We check every item in \Cref{def:2d-valid}. \Cref{def:2d-valid-3} follows immediately from the validity of $\cP$ and the fact that $\sigma'_x = \sigma_x$ whenever $x\in X\backslash R' = X\backslash R$.

Now we prove \Cref{def:2d-valid-1}. All non-$\sigma$-separated points $x\in X\backslash X_+$ have $\sigma'_x = \sigma_x$, and $\ell_x' = \ell_x$, so they satisfy $\ell_x'= \ell_x \ge \|x - \sigma_x\|_\infty = \|x - \sigma_x'\|_\infty$. By \Cref{def:2d-valid-3}, all points in $X\backslash R = X\backslash R'$ also satisfy $\ell_x' \ge \|x - \sigma'_x\|_\infty$. It remains to check \Cref{def:2d-valid-1} for $\sigma$-separated relevant points $x\in R\cap X_+$. By \Cref{lm:2d-separated-relevant}, these points satisfy 
\[
\ell_x' = \ell_x + 2^{11}L(\ell_x/L)^{1/(d - \|t_x\|_0)}\ge \ell_x + q_x \ge \|x - \sigma_x\|_\infty + \|\sigma_x - \eta_x\|_\infty \ge \|x - \eta_x\|_\infty = \|x - \sigma_x'\|_\infty.
\]

We now move on to \Cref{def:2d-valid-2}. It suffices to prove that $\max\{M(\cP_1),M(\cP_2)\}\le M(\cP)$, which follows directly from \Cref{lm:2d-M-star}.

Now we prove \Cref{def:2d-valid-4}. We prove it for $\cP_1$, and omit the similar proof for $\cP_2$. Define $b'_1(j),b'_2(j)$ similarly as $b_1(j),b_2(j)$ are defined in \Cref{def:boundaries} except that we replace $Y$ by $Y_1$.

Suppose $x\in X_1\cap R'$ has $t_x'(j) \ne 0$. If $j = j^*$ and $t_x'(j) = 2$, then by \Cref{lm:2d-type} it must be the case that $x\in R\cap X_+$. We have $\sigma'_x(j) \le b'_2(j) \le \sigma_x(j)$, so 
\[
|x(j) - b'_2(j)| \le \max\{|x(j) - \sigma'_x(j)|, |x(j) - \sigma_x(j)|\} \le \max\{\ell_x',\ell_x\} = \ell_x'. 
\]
If $j \ne j^*$ or $t_x'(j) \ne 2$, we have $t_x'(j) = t_x(j)$. Define $i = t_x'(j) = t_x(j)$. If $i = 1$, we have $b_i(j)\le b_i'(j) \le \sigma_x'(j)$; if $i = 2$, we have $b_i(j)\ge b_i'(j) \ge \sigma_x'(j)$. In both cases,
\[
|x(j) - b'_i(j)| \le \max\{|x(j) - \sigma'_x(j)|, |x(j) - b_i(j)|\} \le \max\{\ell_x',\ell_x\} = \ell_x'. 
\ifdefined\soda
\else
\qedhere
\fi
\]
\end{proof}
\ifdefined\soda
Now we prove \Cref{lm:2d-cost}.
\begin{proof}
\else
\begin{proof}[Proof of \Cref{lm:2d-cost}]
\fi
Since the preprocessing step preserves the validity of $\cP$ and does not increase $A(\cP)$, we assume w.l.o.g.\ that $\cP = (X, Y, (\sigma_x)_{x\in X}, (\ell_x)_{x\in X}, (t_x)_{x\in X})$ is the subproblem \emph{after} the preprocessing step.
Define $M = M(\cP)$.
We have
\begin{equation}
\label{eq:subproblem-cost}
A(\cP_1) + A(\cP_2) = f(M(\cP_1)) + f(M(\cP_2)) + 2^{32}\sum_{x\in R'_1}(\ell_x')^2 + 2^9\sum_{x\in R'_2}(\ell_x')^2 + \sum_{x\in X\backslash R}(\ell_x')^2.
\end{equation}
Moreover,
\begin{align}
\sum_{x\in R'_1} (\ell_x')^2 & = \sum_{x\in R_1\cap R_1'}(\ell_x')^2 + \sum_{x\in R_0\cap R_1'}(\ell_x')^2 \nonumber \\
& = \sum_{x\in R_1\cap R_1'}(\ell_x')^2 + \sum_{x\in R_0\cap X_+}(\ell_x')^2 \nonumber \\
& = \sum_{x\in R_1\cap R_1'}\ell_x^2 + \sum_{x\in R_0\cap X_+}(\ell_x + 2^{11}L(\ell_x/L)^{1/2})^2 \tag{by \eqref{rule:2d-ell}}\\
& \le \sum_{x\in R_1\cap R_1'}\ell_x^2 + 2^{23}\sum_{x\in R_0\cap X_+}\ell_xL \nonumber\\
& \le \sum_{i\in R_1\cap R_1'}\ell_x^2 + 2^{25}M^*\log(M/M^*)\log\log_2(2k),\label{eq:2d-combine-1}
\end{align}
where the last inequality is by \Cref{lm:2d-cutting} and \Cref{def:2d-valid} \Cref{def:2d-valid-2}. Similarly, 
\begin{align}
\sum_{x\in R'_2} (\ell_x')^2 & = \sum_{x\in R_2}(\ell_x')^2 + \sum_{x\in R_1\cap R_2'}(\ell_x')^2\nonumber \\
& = \sum_{x\in R_2}\ell_x^2 + \sum_{x\in R_1\cap R_2'}(\ell_x + 2^{11}\ell_x)^2\tag{by \eqref{rule:2d-ell}}\\
& \le 
\sum_{x\in R_2}\ell_x^2 + 2^{23}\sum_{x\in R_1\cap R_2'}\ell_x^2,\label{eq:2d-combine-2}
\end{align}
Applying \Cref{lm:2d-M-star},
\begin{align}
& f(M(\cP_1)) + f(M(\cP_2)) + 2^{57}M^*\log(M/M^*)\log\log_2(2k) \nonumber \\
= {} & f(M_1^*) + f(M_2^*) + 2^{57}M^*\log(M/M^*)\log\log_2(2k) \nonumber \\
\le {} & f(M^*) + f(M - M^*) + 2^{57}M^*\log(M/M^*)\log\log_2(2k)\nonumber \\
\le {} & 2^{57}\Big(M^*(1 + \log (M^*/m))\log\log_2(2k) + (M - M^*)(1 + \log(M/m))\log\log_2(2k)\nonumber \\
& + M^*\log(M/M^*)\log\log_2(2k)\Big)\nonumber \\
= {} & f(M).\label{eq:2d-combine-3}
\end{align}
Combining the inequalities as $\eqref{eq:2d-combine-1}\times 2^{32} + \eqref{eq:2d-combine-2}\times 2^9 + \eqref{eq:2d-combine-3}$ and simplifying using \eqref{eq:subproblem-cost}, we get
$
A(\cP_1) + A(\cP_2) \le A(\cP)
$,
as desired.
\end{proof}
\subsection{Building a decision tree}
\label{sec:2d-tree}
We prove \Cref{thm:2d} by describing the algorithm $\tdpostprocess$ that takes an arbitrary $k$-clustering $\cC$ and turns it into a $k$-explainable clustering with respect to a decision tree $T$.

\begin{algorithm2e}[t]
 \caption{Algorithm \texttt{post-process\_2d} via \texttt{decision\_tree\_2d}}
 \label{alg:1}
\SetKwInOut{Input}{Input}
\SetKwInOut{Output}{Output}

\SetKwProg{myalg}{Algorithm}{}{end}

\Input{A $k$-clustering $\cC$ of points $x_1,\ldots,x_n$ in $d = 2$ dimensions}
\Output{A $k$-clustering $\cC'$ of points $x_1,\ldots,x_n$, a decision tree $T$}
\BlankLine
\myalg{\tdpostprocess{$x_1,\ldots,x_n;\cC$}}{
	Let $y_1,\ldots,y_k$ be the centroids of $\cC$, and $\xi:[n]\rightarrow [d]$ be the assignment mapping of $\cC$\;
	$X \gets \{x_1,\ldots,x_n\}, Y \gets \{y_1,\ldots,y_k\}$\;
	$\forall i\in [n], \forall x\in X, \forall j\in [d]$, set $\sigma_{x_i} \gets y_{\xi(i)}, \ell_x \gets \|x - \sigma_x\|_\infty, t_x(j) \gets 0$\;
	 $\widetilde \cP \gets (X, Y, (\sigma_x)_{x\in X}, (\ell_x)_{x\in X}, (t_x)_{x\in X})$\tcc*{Initial subproblem}
	$m \gets \frac 1k\sum_{x\in X}\ell_x^2$ \tcc*{Centroid mass}
	$((\delta_x)_{x\in X}, T)\gets \tddecisiontree{$\widetilde \cP$}$\;
	Let $\xi':[n]\rightarrow [d]$ be such that $y_{\xi'(i)} = \delta_{x_i}$.\;
	Let $\cC'$ be the clustering with centroids $y_1,\ldots,y_k$ and assignment mapping $\xi'$\;
	\Return $\cC',T$\;
}

\vspace{\baselineskip}

\Input{A valid subproblem $\cP = (X, Y, (\sigma_x)_{x\in X}, (\ell_x)_{x\in X}, (t_x)_{x\in X})$}
\Output{An assigned centroid $\delta_x\in Y$ for every point $x\in X$, a decision tree $T$}

\BlankLine
\myalg{\tddecisiontree{$\cP$}}{
	Let $L$ be the diameter of $\cP$\;
	\eIf{$L = 0$}{
	Pick an arbitrary $y\in Y$ and set $\delta_x = y$ for all $x\in X$\;
	Let $T$ be the decision tree with a single node\;
	}
	{
		$(\cP_1,\cP_2,j^*,\theta) \gets \tdsinglecut{$\cP$}$\;
		$((\delta_x)_{x\in X_1}, T_1) \gets \tddecisiontree{$\cP_1$}$\;
		$((\delta_x)_{x\in X_2}, T_2) \gets \tddecisiontree{$\cP_1$}$\;
		Define $T$ to be the decision tree with root corresponding to hyperplane $(j^*,\theta)$ and its left and right sub-trees being $T_1$ and $T_2$.\;
		
	}
	\Return $(\delta_x)_{x\in X}, T$\;
}
\end{algorithm2e}

As shown in \Cref{alg:1}, our algorithm \tdpostprocess 
calls an algorithm $\tddecisiontree$ that takes a valid subproblem $\cP = (X, Y, (\sigma_x)_{x\in X}, (\ell_x)_{x\in X}, (t_x)_{x\in X})$ and produces a decision tree $T$ and an assigned centroid $\delta_x\in Y$ for every $x\in X$. Algorithm \tddecisiontree works recursively as follows: if the diameter $L$ of $\cP$ is zero, which means that all centroids $y\in Y$ are at the same location, return $((\delta_x)_{x\in X},T)$, where $\delta_x$ is identical for all $x$ and equals to an arbitrary centroid $y\in Y$, and $T$ is the tree with a single node; else, the algorithm $\tddecisiontree$ calls $\tdsinglecut$ on $\cP$, obtains a hyperplane $(j^*,\theta)$ and two subproblems $\cP_1,\cP_2$. 
The algorithm $\tddecisiontree$ recursively calls itself on $\cP_1$ and $\cP_2$, and obtains $((\delta_x)_{x\in X_1},T_1)$ and $((\delta_x)_{x\in X_2},T_2)$. The algorithm constructs a tree $T$ with root corresponding to $(j^*,\theta)$ and its left and right sub-trees being $T_1$ and $T_2$. The algorithm returns $((\delta_x)_{x\in X},T)$.

\begin{lemma}
\label{lm:2d-induction}
Assuming the input $\cP = (X, Y, (\sigma_x)_{x\in X}, (\ell_x)_{x\in X}, (t_x)_{x\in X})$ to $\tddecisiontree$ is valid.
The output $((\delta_x)_{x\in X}, T)$ of $\tddecisiontree$ satisfies the following properties.
The decision tree $T$ has at most $|Y|$ leaves.
For every leaf $v$ of the decision tree $T$ and
every pair of points $x,x'\in X$ in the region defined by $v$, we have $\delta_x = \delta_x'$. Moreover,
$\sum_{x\in X}\|x - \delta_x\|_\infty^2 \le A(\cP)$.
\end{lemma}
\begin{proof}
We prove the lemma by induction on $|Y|$. When $|Y| = 1$, we have $L = 0$ and thus $\tddecisiontree$ returns without calling $\tdsinglecut$. The lemma is trivial in this case. Now suppose the lemma is true when $|Y| < u$ for an integer $u > 1$, and we prove the lemma for $|Y| = u$. If $L = 0$, then again the lemma is trivial; otherwise, the algorithm $\tddecisiontree$ calls $\tdsinglecut$, and the lemma follows from the induction hypothesis together with \Cref{lm:2d-valid} and \Cref{lm:2d-cost}.
\end{proof}

Given a $k$-clustering $\cC$ of points $x_1,\ldots,x_n$ consisting of centroids $y_1,\ldots,y_k$, and an assignment mapping $\xi:[n]\rightarrow [k]$, the algorithm $\tdpostprocess$
computes an \emph{initial subproblem} $\widetilde \cP$ by setting $X = \{x_1,\ldots,x_n\}, Y = \{y_1,\ldots,y_k\}, \sigma_{x_i} = y_{\xi(i)}, \ell_x = \|x - \sigma_x\|_\infty, t_x(j) = 0, \forall j\in [d]$. Setting the centroid mass $m = \frac 1k\sum_{x\in X}\ell_x^2$, we have $\widetilde \cP$ is valid and $A(\widetilde \cP) = O(\log k\log\log_2 (2k))\cdot \cost(\cC)$. Algorithm $\tdpostprocess$ then calls $\tddecisiontree$ on $\widetilde \cP$ and obtains $(\delta_x)_{x\in X}$ and $T$. Algorithm $\tdpostprocess$ returns the decision tree $T$ and a clustering $\cC'$ with centroids $y_1,\ldots,y_k$ and assignment mapping $\xi':[n]\rightarrow [k]$ such that $y_{\xi'(i)}  = \delta_{x_i}$.
\ifdefined\soda

Now we finish the proof of \Cref{thm:2d}.
\begin{proof}
\else
\begin{proof}[Proof of \Cref{thm:2d}]
\fi
Since $\tdsinglecut$ computes $\theta$ in polynomial time by \Cref{lm:2d-cutting}, the entire algorithm $\tdsinglecut$ can be implemented in polynomial time. 
Consequently, $\tddecisiontree$ and $\tdpostprocess$ both run in polynomial time.
Let $\cC'$ and $T$ be the output of algorithm $\tdpostprocess$. By \Cref{lm:2d-induction},  we know $\cC'$ is $k$-explainable w.r.t.\ $T$, and 
\begin{align*}
\cost(\cC') = \sum_{i=1}^n\|x_i - y_{\xi'(i)}\|_2^2 \le 2 \sum_{i=1}^n\|x_i - y_{\xi'(i)}\|_\infty^2 & =  2\sum_{i=1}^n\|x_i - \delta_{x_i}\|_\infty^2 \\
& \le 2A(\widetilde\cP) \le O(\log k\log\log_2(2k))\cost(\cC).
\end{align*}
It is by definition that $\cC'$ uses the same centroids as $\cC$ does. Choosing $\cC$ as the output of a poly-time constant factor approximation algorithm for $k$-means, e.g.\ \cite{kanungo2004local,MR3734218,grandoni2021refined}, or since $d=2$, a PTAS for $k$-means \cite{MR3630998,MR3630997,MR3775817}, gives the competitive ratio bound $O(\log k\log\log_2(2k))$.
\end{proof}
\section{Explainable $k$-means in $d>2$ dimensions}
\label{sec:d>2}
We now describe our algorithm for higher dimensions, i.e., $d > 2$. The algorithm also works for $d = 2$ despite giving a worse bound than \Cref{thm:2d}. We follow the same structure as the previous section, and emphasize the differences from it. Our goal is to prove the following high-dimensional analogue of \Cref{thm:2d}.
\begin{theorem}
\label{thm:d>2}
Assume $k, d \ge 2$. 
There exists a poly-time algorithm $\postprocess$ that takes a $k$-clustering $\cC$ of $n$ points in $d$ dimensions, and outputs a clustering $\cC'$ of the $n$ points and a decision tree $T$ with at most $k$ leaves such that
\begin{enumerate}
\item $\cC'$ is $k$-explainable with respect to $T$;
\item $\cost(\cC') \le O\big(k^{1 - 2/d}(\log k)^8(\log\log_2(2k))^3d^4\big)\cdot \cost(\cC)$;
\item $\cC'$ uses the same $k$ centroids as $\cC$ does.
\end{enumerate}
Consequently, there exists a poly-time algorithm that takes $n$ points in $d$ dimensions and outputs a $k$-explainable clustering with competitive ratio $O\big(k^{1 - 2/d}(\log k)^8(\log\log_2(2k))^3d^4\big)$.
\end{theorem}

The bound in \Cref{thm:d>2} can be improved when combined with \cite{makarychev2021,gamlath2021nearly,esfandiari2021almost}.

\begin{corollary}[In light of \cite{makarychev2021,gamlath2021nearly,esfandiari2021almost}]
\label{cor:main}
Assume $k,d\ge 2$.
There exists a poly-time algorithm that takes $n$ points in $d$ dimensions and outputs a $k$-explainable clustering with competitive ratio $k^{1 - 2/d}\,\polylog(k)$.
\end{corollary}
\begin{proof}
Use the following algorithm: when $d \le \log k$, invoke the algorithm in \Cref{thm:d>2} to achieve competitive ratio 
\[
O\big(k^{1 - 2/d}(\log k)^{8}(\log\log_2(2k))^3d^4\big) = k^{1 - 2/d}\,\polylog(k);
\]
when $d > \log k$, invoke the algorithm in \cite{makarychev2021,gamlath2021nearly,esfandiari2021almost} to achieve competitive ratio
\[
k\,\polylog(k) = 
k^{2/d}k^{1 - 2/d}\,\polylog (k) = O(1)\cdot k^{1 - 2/d}\,\polylog (k) = k^{1 - 2/d}\,\polylog (k).\ifdefined\soda\else\qedhere\fi
\]
\end{proof}
In the rest of the section, we assume $k,d \ge 2$.

\subsection{Subproblem}
Compared to subproblems for $d = 2$ defined in \Cref{sec:2d-subproblem},
subproblems for $d > 2$ contain more information:
\begin{definition}[Subproblem for $d > 2$] 
Given points $x_1,\ldots,x_n\in\bR^d$ and centroids $y_1,\ldots,y_k\in\bR^d$, besides what is included in \Cref{def:2d-subproblem},
a subproblem $\cP$ consists of the following in addition:
\begin{enumerate}

\item A color $c_x\in\{-1, 0, \ldots,\lfloor\log_2 k\rfloor - 1\}$ for every point $x\in X$. 
\item A scale $s_x\in [1,+\infty)$ for every point $x\in X$;
\item A potential $p_x\in (0, +\infty)$ for every point $x\in X$.
\end{enumerate}
\end{definition}

We fix a positive real number $m$ as the \emph{centroid mass} and define two quantities $M(\cP)$ and $A(\cP)$ for a subproblem $\cP$ with respect to the centroid mass $m$.
This is similar to \Cref{def:cost} but we take the potentials $p_x$ into account:
\begin{definition}
Given a subproblem $\cP$, we define the following quantities:
\begin{align*}
M(\cP) & := m|Y| + \sum_{x\in R} p_x\ell_x^2\big(16(\log (2k))^2\log\log_2 (2k)\big)^{2 - \|t_x\|_0},\\
A(\cP) & := f(M(\cP)) + \sum_{x\in X\backslash R} p_x\ell_x^2,
\end{align*}
where
\[
f(M) := 16M(M/m)^{1/\log (2k)}(1 + \log (M/m))\log\log_2(2k).
\]
\end{definition}
We define subproblem boundaries in the same way as \Cref{def:boundaries}. Our definition for valid subproblems includes more requirements than \Cref{def:valid}:
\begin{definition}[Valid subproblem]
\label{def:valid}
Given points $x_1,\ldots,x_n\in\bR^d$ and centroids $y_1,\ldots,y_k\in\bR^d$,
a subproblem $\cP = (X, Y, (\sigma_x)_{x\in X}, (\ell_x)_{x\in X}, (t_x)_{x\in X},(c_x)_{x\in X}, (s_x)_{x\in X}, (p_x)_{x\in X})$ 
is valid if in addition to the requirements in \Cref{def:2d-valid}, it satisfies all of the following:
\begin{enumerate}
\item \label{def:valid-t} For all $x\in R$, $\|t_x\|_0\le 2$. (Thus, $R = R_0 \cup R_1 \cup R_2$.)
\item \label{def:valid-power} If $x\in R_0$, then $\ell_x$ is either zero or a power of $2$, i.e., $\ell_x = 0$ or $\ell_x = 2^a$ for an integer $a$.
\item \label{def:valid-s} If $x\in R_0$, then $s_x = k$.
\item \label{def:valid-color} For every relevant point $x\in R$, its color $c_x = -1$ if and only if $x\in R_0$.
\item \label{def:valid-identical} For every color $c\in\{0,\ldots,\lfloor \log_2k\rfloor - 1\}$ and every relevant type $t:[d]\rightarrow \{0,1,2\}$, define $R_{c,t} = \{x\in R:c_x = c, t_x = t\}$ and $Y_{c,t} = \{\sigma_x:x\in R_{c,t}\}$. If $R_{c,t}\ne \emptyset$, there exists scale $s_{c,t}\ge 1$ and length $\ell_{c,t}\ge 0$ such that all points $x\in R_{c,t}$ have $s_x = s_{c,t}$ and $\ell_x = \ell_{c,t}$. Moreover, $|Y_{c,t}|\le s_{c,t}$.
\end{enumerate}
\end{definition}

\subsection{Making a single cut}
\label{sec:single-cut}
As in \Cref{sec:2d-single-cut}, we describe an efficient algorithm $\singlecut$ that takes a valid subproblem $\cP = (X, Y, (\sigma_x)_{x\in X}, (\ell_x)_{x\in X}, (t_x)_{x\in X}, (c_x)_{x\in X}, (s_x)_{x\in X}, (p_x)_{x\in X})$, and produces two smaller valid subproblems $\cP_1$ and $\cP_2$ together with an axis-parallel hyperplane $(j^*,\theta)$ that separates them. The two subproblems are defined by new assignments $(\sigma'_x)_{x\in X}$, new lengths $(\ell_x')_{x\in X}$, new types $(t_x')_{x\in X}$, new colors $(c_x')_{x\in X}$, new scales $(s_x')_{x\in X}$, and new potentials $(p_x')_{x\in X}$ as follows
\begin{align}
\cP_1 & = (X_1, Y_1, (\sigma_x')_{x\in X_1}, (\ell_x')_{x\in X_1}, (t_x')_{x\in X_1}, (c_x')_{x\in X_1}, (s_x')_{x\in X_1}, (p_x')_{x\in X_1}),\notag\\
\cP_2 & = (X_2, Y_2, (\sigma_x')_{x\in X_2}, (\ell_x')_{x\in X_2}, (t_x')_{x\in X_2}, (c_x')_{x\in X_2}, (s_x')_{x\in X_2}, (p_x')_{x\in X_2}).
\label{eq:subproblems-d>2}
\end{align}
Again, we choose $j^* = \argmax_{j\in [d]}(b_2(j) - b_1(j))$, and define $X_1,X_2,Y_1,Y_2,X_+,X_{11},X_{22}$ as in \eqref{eq:partition}, \eqref{eq:separated} and \eqref{eq:non-separated}.
Throughout \Cref{sec:single-cut},
we assume that the input subproblem $\cP$ is valid 
and has diameter $L > 0$
even when we do not explicitly state so. At the end of the section, we prove the following lemmas for the algorithm $\singlecut$.
\begin{lemma}
\label{lm:valid}
The two new subproblems $\cP_1,\cP_2$ output by $\singlecut$ are both valid.
\end{lemma}
\begin{lemma}
\label{lm:cost}
The two new subproblems $\cP_1,\cP_2$ output by $\singlecut$ satisfy $A(\cP_1) + A(\cP_2) \le A(\cP)$.
\end{lemma}
We describe the algorithm $\singlecut$ step by step in the following subsections.
\subsubsection{Preprocessing}
If $x\in R$ has $\ell_x\ge L/64$, we replace $\ell_x$ by $65\ell_x$, replace $p_x$ by $p_x/65^2$, and set $t_x = \bot$ (thus removing $x$ from $R$). 
Similarly to \Cref{sec:2d-preprocessing}, it is clear that the new subproblem is still valid, and in the new subproblem every point $x\in R$ satisfies $\ell_x\le L/64$. Moreover, the value of $A(\cP)$ does not increase.
For the rest of \Cref{sec:single-cut}, we use $\cP = (X, Y, (\sigma_x)_{x\in X}, (\ell_x)_{x\in X}, (t_x)_{x\in X}, (c_x)_{x\in X}, (s_x)_{x\in X}, (p_x)_{x\in X})$ to denote the subproblem \emph{after} the preprocessing step.

\subsubsection{Forbidding}
Similarly to \Cref{sec:2d-forbidding}, we specify a subset $F$ of the interval $(b_1(j^*),b_2(j^*))$ as the forbidden region.

For all $x\in X$, we define $Y(x),\eta_x, q_x,W_x$ as in \eqref{eq:candidate-target},\eqref{eq:target}, and \eqref{eq:Wx}.
We also define $S, T$ as in \eqref{eq:T}.

For a point $x\in R_0$ with $\ell_x \ge L/32k$, we change its color from $c_x = -1$ (guaranteed by \Cref{def:valid-color} in \Cref{def:valid}) to $c_x = \lfloor\log_2(\ell_x/(L/32k))\rfloor$.
Our preprocessing guarantees that $\ell_x \le L/64$, so the new $c_x$ is an integer between $0$ and $\lfloor \log_2 k\rfloor - 1$. 
Moreover, points $x\in R_0$ have $\ell_x$ being a power of $2$ by \Cref{def:valid-power} in \Cref{def:valid}. This means that if $x,x'\in R_0$ have $c_x = c_{x'} \ge 0$ after the color change, then $\ell_x = \ell_{x'}$.
By \Cref{def:valid-s} in \Cref{def:valid},
all points $x\in R_0$ have $s_x = k$,
so \Cref{def:valid-identical} in \Cref{def:valid} still holds with the new colors and the induced new definitions of $R_{c,t},Y_{c,t}, s_{c,t},\ell_{c,t}$.

For every color $c\ge 0$ and every type $t\ne \bot$, if $R_{c,t} = \emptyset$, define $H_{c,t} = \emptyset$. If $R_{c,t}\ne \emptyset$, we know $Y_{c,t}\ne \emptyset$,
and we define $H_{c,t}$ as follows, where
$E_{c,t}(y)$ is the interval $[y(j^*) - \ell_{c,t}, y(j^*) + \ell_{c,t}]$ for every $y\in Y_{c,t}$,
and $\one_{E}(\cdot)$ denotes the indicator function of $E\subseteq \bR$:
\[
H_{c,t} = \left\{\theta\in (b_1(j^*), b_2(j^*)): \sum_{y\in Y_{c,t}}\one_{E_{c,t}(y)}(\theta) > 48\cdot 2^{\|t\|_0}{d\choose \|t\|_0}s_{c,t}\ell_{c,t}(\log_2 k)/L\right\}.
\]

The entire forbidden region is
\begin{align*}
F = {} & 
(b_1(j^*), b_1(j^*) + L/32]
\cup[b_2(j^*) - L/32, b_2(j^*))\\
& \cup \bigcup_{y\in Y}[y(j^*) - L/32k, y(j^*) + L/32 k]\cap (b_1(j^*),b_2(j^*))\\
& \cup \bigcup_{x\in T}W_x\\
& \cup \bigcup_{c\ge 0, t\ne \bot}H_{c,t}.
\end{align*}
Similarly to the $d = 2$ case, $F$ can be represented as a union of finitely many \emph{disjoint} intervals, and the representation can be computed in poly-time. The following lemma has essentially the same proof as \Cref{lm:2d-type}:
\begin{lemma}
\label{lm:type}
If we choose $\theta\in (b_1(j^*),b_2(j^*))\backslash F$, then every $x\in R$ with $t_x(j^*) = 1$ belongs to $X_{11}$, and similarly every $x\in R$ with $t_x(j^*) = 2$ belongs to $X_{22}$. Consequently, every $x\in R\cap X_+$ satisfies $t_x(j^*) = 0$.
\end{lemma}
Similarly to \Cref{lm:2d-separated-relevant} and \Cref{lm:forbidden}, we prove \Cref{lm:separated-relevant} and \Cref{lm:forbidden-d>2} below.
\begin{lemma}
\label{lm:separated-relevant}
If we choose $\theta\in (b_1(j^*),b_2(j^*))\backslash F$, then every $\sigma$-separated relevant point $x\in R\cap X_+$ satisfies all of the following:
\begin{enumerate}
\item \label{lm:separated-relevant-1} $x(j^*)\ne \sigma_x(j^*)$;
\item \label{lm:separated-relevant-2} $\ell_x\ge L/64k$ (and thus $c_x \ge 0$);
\item \label{lm:separated-relevant-3} $t_x(j^*) = 0$;
\item \label{lm:separated-relevant-4} if $x\in R_0\cup R_1$, then $\frac{q_x}{L}\le 2^{11}(\frac{\ell_x}{L})^{1/(d - \|t_x\|_0)}$.
\end{enumerate}
Moreover, for every color $c\ge 0$ and every type $t\ne \bot$, 
\begin{equation}
\label{eq:new-s}
|\{\sigma_x:x\in R_{c,t}\cap X_+\}|\le 48\cdot 2^{\|t\|_0}{d\choose \|t\|_0}s_{c,t}\ell_{c,t}(\log_2 k)/L.
\end{equation}
\end{lemma}
\begin{proof}
\Cref{lm:separated-relevant-3} follows directly from \Cref{lm:type}.
\Cref{lm:separated-relevant-1} and \Cref{lm:separated-relevant-4} follow from the same argument in the proof of \Cref{lm:2d-separated-relevant}.

Assume for the sake of contradiction that \Cref{lm:separated-relevant-2} does not hold for $x\in R\cap X_+$. Then by \Cref{def:2d-valid} \Cref{def:2d-valid-1}, $|x(j^*) - \sigma_x(j^*)|\le \|x - \sigma_x\|_\infty \le \ell_x < L/64k$. Therefore, 
\[
W_x \subseteq [\sigma_x(j^*) - L/64k, \sigma_x(j^*) + L/64k]\cap (b_1(j^*), b_2(j^*))\subseteq F.
\]
However, the fact that $x\in X_+$ implies $\theta\in W_x$, and thus $\theta\in F$, a contradiction.

Assume for the sake of contradiction that \eqref{eq:new-s} does not hold for color $c\ge 0$ and type $t\ne \bot$. For every $y\in \{\{\sigma_x:x\in R_{c,t}\cap X_+\}\}\subseteq Y_{c,t}$, we have $\theta\in E_{c,t}(y)$. Since \eqref{eq:new-s} is violated, this means $\theta\in H_{c,t}\subseteq F$, a contradiction.
\end{proof}
\begin{lemma}
\label{lm:forbidden-d>2}
The forbidden region $F$ has length at most $L/2$.
\end{lemma}
\begin{proof}
Using a similar argument to the proof of \Cref{lm:volume}, we have $|\bigcup_{x\in T}W_x| \le L/4$. 
It suffices to prove $|\bigcup_{c\ge 0, t\ne \bot}H_{c,t}|\le L/8$.

For a uniform random $\theta\in(b_1(j^*), b_2(j^*))$, we have $\bE[\one_{E_{c,t}(y)}(\theta)] \le |E_{c,t}(y)|/ L = 2\ell_{c,t}/L$. Therefore,
\[
\bE[\sum_{y\in Y_{c,t}}\one_{E_{c,t}(y)}(\theta)] \le |Y_{c,t}|\cdot 2\ell_{c,t}/L \le 2s_{c,t}\ell_{c,t}/L.
\]
By Markov's inequality, $|H_{c,t}|\le L/\left(24\cdot 2^{\|t\|_0}{d\choose \|t\|_0}\log_2 k\right)$.
Summing up over $c,t$, we have
\begin{align*}
\sum_{c\ge 0, t\ne \bot}|H_{c,t}| 
& \le \sum_{i = 0}^2\sum_{\|t\|_0 = i}\sum_{c=0}^{\lfloor \log_2 k \rfloor - 1}L/\left(24\cdot 2^{i}{d\choose i}\log_2 k\right)\\
& \le \sum_{i = 0}^2\sum_{\|t\|_0 = i}L/\left(24\cdot 2^{i}{d\choose i}\right)\\
& = \sum_{i=0}^2 L /24\\
& = L/8.
\ifdefined\soda
\else
\qedhere
\fi
\end{align*}
\end{proof}

\subsubsection{Cutting}
We choose $\theta\in(b_1(j^*), b_2(j^*))$ in a similar way as in \Cref{sec:2d-cutting} based on \Cref{lm:cutting} below. For every choice of $\theta$, define
\begin{align*}
M_1^* & = m|Y_1| + \sum_{x\in R\cap X_{11}} p_x\ell_x^2\big(16(\log (2k))^2\log\log_2 (2k)\big)^{2 - \|t_x\|_0},\\
M_2^* & = m|Y_2| + \sum_{x\in R\cap X_{22}} p_x\ell_x^2\big(16(\log (2k))^2\log\log_2 (2k)\big)^{2 - \|t_x\|_0},\\
M^* &= \min\{M(\cP)/2, M(\cP) - M_1^*, M(\cP) - M_2^*\}.
\end{align*}
\begin{lemma}
\label{lm:cutting}
There exists $\theta\in (b_1(j^*),b_2(j^*))\backslash F$ satisfying
\[
\sum_{x\in R\cap X_+}p_x\ell_xL\big(16(\log (2k))^2\log\log_2 (2k)\big)^{2 - \|t_x\|_0} \le 8M^*\log (M(\cP)/M^*)\log\log_2(M(\cP)/m).
\]
Moreover, $\theta$ can be computed in poly-time.
\end{lemma}
We omit the proof as it is essentially the same as the proof of \Cref{lm:2d-cutting}.

\subsubsection{Updating}
Having computed the hyperplane $(j^*,\theta)$, we get the partions $(X_1,X_2)$ and $(Y_1,Y_2)$ by \eqref{eq:partition}.
We now specify the new assignment $\sigma'_x$, new lengths $\ell_x'$, new types $t_x'$, new colors $c_x'$, new scales $s_x'$, and new potentials $p_x'$. The two new subproblems $\cP_1,\cP_2$ can then be formed by \eqref{eq:subproblems-d>2}.

For every non-$\sigma$-separated point $x\in X\backslash X_+$ we define $\sigma'_x = \sigma_x$, $\ell_x' = \ell_x$, $t_x' = t_x$, $s_x' = s_x$, $p_x' = p_x$, and define $c_x' = -1$ if $x\in R_0$ and $c_x' = c_x$ otherwise. For every $\sigma$-separated irrelevant point $x\in X_+\backslash R$, we define $\ell_x' = \ell_x$, $t_x' = t_x(=\bot)$, $s_x' = s_x$, $p_x' = p_x$, $c_x' = c_x$, and 
define $\sigma'_x$ to be an arbitrary centroid in $Y$ that lies on the same side of the hyperplane $(j^*,\theta)$ with $x$.

It remains to consider relevant points that are $\sigma$-separated, i.e. points $x\in R\cap X_+$. 
We define 
\begin{equation}
\label{rule:s}
s_x' = 48\cdot 2^{\|t_x\|_0}{d\choose\|t_x\|_0}s_x\ell_x(\log_2 k)/L.
\end{equation}
It is clear that $s_x'\ge 1$ because otherwise no point in $R_{c_x,t_x}$ should be $\sigma$-separated by \Cref{lm:separated-relevant} \Cref{lm:separated-relevant-2} and \eqref{eq:new-s}. Define $\sigma'_x$ to be the centroid $y\in Y$ with the minimum $\|y - \sigma_x\|_\infty$ that lies on the same side of the hyperplane $(j^*,\theta)$ with $x$. This ensures that points $x\in R\cap X_+\cap X_1$ with the same $\sigma_x$ have the same $\sigma'_x$, which holds similarly for points $x\in R\cap X_+ \cap X_2$. Since every centroid in $Y(x)$ is a candidate for $\sigma'_x$, we have $\|\sigma_x - \sigma'_x\|_\infty \le q_x$. We define $c_x' = c_x$. The definition for $t_x',\ell_x', p_x'$ depends on whether $x\in R_0\cup R_1$ or $x\in R_2$ as follows.

If $x\in (R_0\cup R_1)\cap X_+$, define $t_x'$ to be equal to $t_x$, except that we change $t_x'(j^*)$ to either $1$ or $2$ from the original value $t_x(j^*) = 0$ (\Cref{lm:separated-relevant} \Cref{lm:separated-relevant-3}). Specifically, define $t_x'(j^*) = 1$ if $x\in X_2$, and $t_x'(j^*) = 2$ if $x\in X_1$. Define
\begin{equation}
\label{rule:ell}
\ell_x'= 2^{12}L(\ell_x/L)^{1/(d - \|t_x\|_0)}.
\end{equation}
Define $p_x'$ so that $p_x'(\ell_x')^2 = p_x\ell_xL$, or equivalently,
\begin{equation}
\label{rule:p1}
p_x' = p_x(\ell_x/L)^{1 - 2/(d - \|t_x\|_0)}/2^{24}.
\end{equation}

If $x\in R_2\cap X_+$, define $t_x' = \bot$ and $\ell_x' = 2L$. Again, define $p_x'$ so that $p_x'(\ell_x')^2 = p_x\ell_xL$, or equivalently,
\begin{equation}
\label{rule:p2}
p_x' = p_x(\ell_x/L)/4.
\end{equation} 

This completes our definition of $\sigma_x', \ell_x', t_x', c_x', s_x'$ and $p_x'$. The algorithm $\singlecut$ outputs the two new subproblems $\cP_1$ and $\cP_2$ formed according to \eqref{eq:subproblems-d>2} together with the hyperplane $(j^*,\theta)$. Before we prove \Cref{lm:valid} and \Cref{lm:cost}, we need some inequalities about $M(\cP_1),M(\cP_2),M_1^*,M_2^*$, and $M^*$.
\begin{lemma}
\label{lm:M-star-d>2}
We have the following inequalities:
\begin{align}
M(\cP_1) \ge M_1^*, & \quad \textnormal{and}\quad M(\cP_2) \ge M_2^*,\label{eq:M-star-d>2-1}\\
\min\{M_1^*,M_2^*\}\le M^*, & \quad \textnormal{and}\quad \max\{M_1^*,M_2^*\}\le M(\cP) - M^*.\label{eq:M-star-d>2-2}
\end{align}
Moreover,
\begin{equation}
\label{eq:M-star-d>2-3}
M(\cP_1) + M(\cP_2) - M_1^* - M_2^*  \le M^*/2\log (2k).
\end{equation}
\end{lemma}
\begin{proof}
Inequalities \eqref{eq:M-star-d>2-2} follow from a similar argument to the proof of \eqref{eq:2d-M-star-2} in \Cref{lm:2d-M-star}.
Inequalities \eqref{eq:M-star-d>2-1} are proved as follows based on our update rules:
\begin{align*}
M(\cP_1) - M_1^* & = \sum_{x\in (R_0 \cup R_1)\cap X_+ \cap X_1}p_x'(\ell_x')^2\big(16\log (2k)\log\log_2(2k)\big)^{2 - \|t_x'\|_0}\ge 0,\\
M(\cP_2) - M_1^* & = \sum_{x\in (R_0 \cup R_1)\cap X_+ \cap X_2}p_x'(\ell_x')^2\big(16\log (2k)\log\log_2(2k)\big)^{2 - \|t_x'\|_0}\ge 0.
\end{align*}
Summing up the above inequalities, we get \eqref{eq:M-star-d>2-3}:
\begin{align*}
& M(\cP_1) + M(\cP_2) - M_1^* - M_2^* \\
= {} & \sum_{x\in (R_0\cup R_1)\cap X_+}p_x'(\ell_x')^2\big(16\log (2k)\log\log_2(2k)\big)^{2 - \|t_x'\|_0}\\
= {} & \sum_{x\in (R_0\cup R_1)\cap X_+}p_x\ell_x L\big(16\log (2k)\log\log_2(2k)\big)^{2 - \|t_x'\|_0}\\
= {} & \sum_{x\in (R_0\cup R_1)\cap X_+}p_x\ell_x L\big(16\log (2k)\log\log_2(2k)\big)^{2 - (\|t_x\|_0 + 1)}\\
\le {} & \big(16(\log (2k))^2\log\log_2 (2k)\big)^{-1}\sum_{x\in R\cap X_+}p_x\ell_xL\big(16\log (2k)\log\log_2(2k)\big)^{2 - \|t_x\|_0}\\
\le {} & M^*/2\log (2k). 
\end{align*}
The last inequality uses \Cref{lm:cutting}, $M(\cP)/m\le 2k$ by \Cref{def:2d-valid} \Cref{def:2d-valid-2}, and $M^*\ge \min\{M_1^*,M_2^*\}\ge m$ by \eqref{eq:M-star-d>2-2}.
\end{proof}
We conclude \Cref{sec:single-cut} by proving \Cref{lm:valid} and \Cref{lm:cost}.
\ifdefined\soda
We first prove \Cref{lm:valid}.
\begin{proof}
\else
\begin{proof}[Proof of \Cref{lm:valid}]
\fi
Let $\cP = (X, Y, (\sigma_x)_{x\in X}, (\ell_x)_{x\in X}, (t_x)_{x\in X}, (c_x)_{x\in X}, (s_x)_{x\in X}, (p_x)_{x\in X})$ denote the valid subproblem \emph{after} the preprocessing step.

We first check every item in \Cref{def:2d-valid}. \Cref{def:2d-valid-3} follows from a similar argument to the proof of \Cref{lm:2d-valid}. The only difference is that we need to consider points $x\in R_2\cap X_+$ separately. For these points, we have $t_x' = \bot$. By our preprocessing step, we have $\ell_x \le L/64$, and thus for all $y\in Y$,
\[
\ell_x' = 2L \ge L + \ell_x \ge \|\sigma_x - y\|_\infty + \|x - \sigma_x\|_\infty \ge \|x - y\|_\infty. 
\]
\Cref{def:2d-valid-1} also follows from a similar argument to the proof of \Cref{lm:2d-valid}, noting that for $x\in (R_0\cup R_1)\cap X_+$, update rule \eqref{rule:ell} implies
\[
\ell'_x \ge \ell_x + 2^{11}L(\ell_x/L)^{1/(d - \|t_x\|_0)}\ge \ell_x + q_x \ge \ell_x + \|\sigma_x - \sigma'_x\|_\infty \ge \|x - \sigma'_x\|_\infty.
\]
Defining $M = M(\cP)$,
\Cref{def:2d-valid-2} holds because of \Cref{lm:M-star-d>2}:
\[
\max\{M(\cP_1), M(\cP_2)\}\le \max\{M_1^*,M_2^*\} + M^*/2\log (2k) \le M - M^* + M^*/2\log (2k) \le M \le 2km.
\]
\Cref{def:2d-valid} \Cref{def:2d-valid-4} follows from the same argument as the proof of \Cref{lm:2d-valid}.

We now check every item in \Cref{def:valid}. \Cref{def:valid-t} is clear from our update rules. \Cref{def:valid-power}, \Cref{def:valid-s} and \Cref{def:valid-color} follow from the fact that $\|t_x'\|_0 = 0$ if and only if $x\in R_0\backslash X_+$.

We prove \Cref{def:valid} \Cref{def:valid-identical} for $\cP_1$, and omit the similar proof for $\cP_2$. For color $c\ge 0$ and type $t\ne \bot$, define $R'_{c,t} = \{x\in X_1:c_x' = c, t_x' = t\}$.

If $t(j^*)\ne 2$, we know $R'_{c,t}\cap X_+ = \emptyset$, so $R'_{c,t}\subseteq R_{c,t} \backslash X_+$. Therefore, for $x\in R'_{c,t}$, we have $s_x' = s_x = s_{c,t}, \ell_x' = \ell_x = s_{c,t}, \sigma'_x = \sigma_x$. \Cref{def:valid} \Cref{def:valid-identical} is trivial in this case.

If $t(j^*) = 2$, by \Cref{lm:type}, we know $R'_{c,t}\subseteq (R_0\cup R_1)\cap X_+$, and $R'_{c,t}\subseteq R_{c,\widetilde t}$, where $\widetilde t$ is equal to $t$ except that $\widetilde t(j^*) = 0$. 
By the update rules \eqref{rule:s} and \eqref{rule:ell}, 
all points $x\in R'_{c,t}$ have 
\begin{align}
\label{eq:ell-group}
\ell_x' & = 2^{12}L(\ell_{c,\widetilde t}/L)^{1/(d - \|\widetilde t\|_0)}, \\
\label{eq:s-group}
s_x' & = 48\cdot 2^{\|\widetilde t\|_0}{d\choose\|\widetilde t\|_0}s_{c,\widetilde t}\ell_{c,\widetilde t}(\log_2 k)/L. & & 
\end{align}
Denote the right-hand-sides of \eqref{eq:ell-group} and \eqref{eq:s-group} as $\ell_{c,t}'$ and $s_{c,t}'$, respectively.
We have
\[
|\{\sigma'_x:x\in R'_{c,t}\}| \le |\{\sigma'_x:x\in R_{c,\widetilde t}\cap X_+\cap X_1\}| \le |\{\sigma_x:x\in R_{c,\widetilde t}\cap X_+ \cap X_1\}| \le s_{c,t}',
\]
where the second inequality is because points in $R\cap X_+\cap X_1$ with the same $\sigma_x$ have the same $\sigma'_x$, and the last inequality is by \eqref{eq:new-s} and \eqref{eq:s-group}.
\end{proof}
\ifdefined\soda
Now we prove \Cref{lm:cost}.
\begin{proof}
\else
\begin{proof}[Proof of \Cref{lm:cost}]
\fi
Since the preprocessing step preserves the validity of $\cP$ and does not increase $A(\cP)$, we assume w.l.o.g.\ that $\cP = (X, Y, (\sigma_x)_{x\in X}, (\ell_x)_{x\in X}, (t_x)_{x\in X}, (c_x)_{x\in X}, (s_x)_{x\in X}, (p_x)_{x\in X})$ is the subproblem \emph{after} the preprocessing step.
Assume w.l.o.g.\ $M_1^*\le M_2^*$. 
Define $M = M(\cP),M_1= M^* + M(\cP_1) - M_1^*$, and $M_2 = (M - M^*) + M(\cP_2) - M_2^*$. 
By \Cref{lm:M-star-d>2},
\begin{align}
& M^* \le M_1 \le M^* + M^*/2\log (2k) \le 3M^*/2 \le 3M/4, \label{eq:d>2-combine-0} \\
& M^* \le M - M^* \le M_2, \label{eq:d>2-combine-1}\\
& m \le M(\cP_1) \le M_1, \quad\textnormal{and} \label{eq:d>2-combine-1.5}\\
& m \le M(\cP_2) \le M_2. \label{eq:d>2-combine-2}
\end{align}
\Cref{lm:M-star-d>2}, \eqref{eq:d>2-combine-0} and \eqref{eq:d>2-combine-1} give us
\begin{equation}
\label{eq:d>2-combine-3}
M_1 + M_2 \le M + M^*/2\log (2k) \le M + \min\{M_1, M_2\}/2\log (2k).
\end{equation}
Define $R' = \{x\in X:t'_x \ne \bot\}$. We have
\begin{align*}
\sum_{x\in X\backslash R'}p_x'(\ell_x')^2  - \sum_{x\in X\backslash R}p_x\ell_x^2 & = \sum_{x\in R_2\cap X_+}p_x'(\ell_x')^2\\
& = \sum_{x\in R_2\cap X_+}p_x\ell_xL\\
& \le \sum_{x\in R\cap X_+}p_x\ell_xL\big(16\log (2k)\log\log_2(2k)\big)^{2 - \|t_x\|_0}\\
& \le 8M^*\log (M/M^*) \log\log_2(M/m) \tag{by \Cref{lm:cutting}}\\
& \le 16M_1\log (M/M_1) \log\log_2(M/m) \tag{by \eqref{eq:d>2-combine-0} and \Cref{claim:monotone-relaxed}}\\
& \le 16M_1\log (M/M_1) \log\log_2(2k). \tag{by \Cref{def:2d-valid} \Cref{def:2d-valid-2}}
\end{align*}
Therefore,
\begin{align*}
& f(M(\cP_1)) + f(M(\cP_2)) + \sum_{x\in X\backslash R'}p_x'(\ell_x')^2  - \sum_{x\in X\backslash R}p_x\ell_x^2\\
\le {} & f(M_1) + f(M_2) + 16M_1\log (M/M_1) \log\log_2(2k) \tag{by \eqref{eq:d>2-combine-1.5}, \eqref{eq:d>2-combine-2}}\\
\le {} & 16(M_1(M_1/m)^{1/\log (2k)} + M_2(M_2/m)^{1/\log (2k)})(1 + \log (M/m))\log\log_2(2k)\\
\le {} & 16M(M/m)^{1/\log (2k)}(1 + \log (M/m))\log\log_2(2k) \tag{by \eqref{eq:d>2-combine-3} and \Cref{claim:super-convexity}}\\
= {} & f(M).
\end{align*}
Rearranging the inequality above,
\[
A(\cP_1) + A(\cP_2) = f(M(\cP_1)) + f(M(\cP_2)) + \sum_{x\in X\backslash R'}p_x'(\ell_x')^2 \le f(M) + \sum_{x\in X\backslash R}p_x(\ell_x)^2 = A(\cP).\ifdefined\soda\else\qedhere\fi
\]
\end{proof}
\subsection{Building a decision tree}
The algorithm $\postprocess$ we use to prove \Cref{thm:d>2} is similar to $\tdpostprocess$ in \Cref{sec:2d-tree}. We first construct an algorithm $\decisiontree$ similar to $\tddecisiontree$, except that it takes a subproblem $\cP$ in $d\ge 2$ dimensions and invokes the algorithm $\singlecut$ we developed in \Cref{sec:single-cut} instead of $\tdsinglecut$.
The algorithm $\postprocess$ calls $\decisiontree$ on an initial subproblem $\widetilde\cP$ with centroid mass $m$, which we define as follows.

Given a $k$-clustering $\cC$ of points $x_1,\ldots,x_n$ consisting of centroids $y_1,\ldots,y_k$, and an assignment mapping $\xi:[n]\rightarrow [k]$, the algorithm $\postprocess$
computes the initial subproblem $\widetilde \cP$ by setting $X = \{x_1,\ldots,x_n\}, Y = \{y_1,\ldots,y_k\}, \sigma_{x_i} = y_{\xi(i)}, t_x(j) = 0, \forall j\in [d], s_x = k, c_x = -1$, and 
\begin{align}
\ell_x & = \left\{\begin{array}{ll}2^{\lceil \log_2 \|x - \sigma_x\|_\infty\rceil}, & \textnormal {if}\ \|x - \sigma_x\|_\infty > 0;\nonumber \\
0, & \textnormal{if}\ \|x - \sigma_x\|_\infty = 0;\end{array}\right.\\
p_x & = 2^{54}k^{1 - 2/d}d^3(48\log_2 k)^3. 
\label{eq:p-initial}
\end{align}
Set $m = \frac 1k\sum_{x\in X}p_x\ell_x^2\big(16(\log (2k))^2\log\log_2 (2k)\big)^{2}$ so that $M(\widetilde \cP)/m = 2k$. It is clear that $\widetilde \cP$ is valid and 
\begin{equation}
\label{eq:AP}
A(\widetilde \cP) = O(k^{1 - 2/d}(\log k)^8(\log\log_2(2k))^3d^3)\cdot \cost(\cC).
\end{equation}
After obtaining the output $((\delta_x)_{x\in X},T)$ from $\decisiontree$, $\postprocess$ returns the decision tree $T$ and a clustering $\cC'$ with centroids $y_1,\ldots,y_k$ and assignment mapping $\xi':[n]\rightarrow [k]$ such that $y_{\xi'(i)}  = \delta_{x_i}$.

Before we prove \Cref{thm:d>2}, we need the following lemma showing that the potential of a point never drops below $1$:

\begin{lemma}
\label{lm:p}
In the process of running algorithm $\postprocess$, 
whenever the algorithm $\decisiontree$ is called, the input subproblem $\cP = (X, Y, (\sigma_x)_{x\in X}, (\ell_x)_{x\in X}, (t_x)_{x\in X}, (c_x)_{x\in X}, (s_x)_{x\in X}, (p_x)_{x\in X})$ to  $\decisiontree$ is valid and satisfies $\forall x\in X,p_x\ge 1$.
\end{lemma}
\begin{proof}
The validity of $\cP$ follows from an induction using the validity of $\widetilde \cP$ and \Cref{lm:valid}. Below we prove $\forall x\in X, p_x\ge 1$.

Fix a point $x_v$ for $v\in [n]$. For $i \in\{0,1,2\}$
consider the moments when $\decisiontree$ is called with a subproblem $\cP = (X, Y, (\sigma_x)_{x\in X}, (\ell_x)_{x\in X}, (t_x)_{x\in X}, (c_x)_{x\in X}, (s_x)_{x\in X}, (p_x)_{x\in X})$ that satisfy $x_v\in X$ and $\|t_{x_v}\|_0 = i$. These moments may not exist, or there may be multiple such moments. However, as long as there is at least one such moment, 
the tuple $(s_{x_v},p_{x_v},\ell_{x_v},t_{x_v})$ must be identical for all such moments, because algorithm $\singlecut$ always keeps $(s_x',p_x',\ell_x',t_x')$ equal to $(s_x,p_x,\ell_x,t_x)$ except when $t_x' = \bot\ne t_x$ or $\|t_x'\|_0 = \|t_x\|_0 + 1$. We use $(s(i), p(i), \ell(i))$ to denote the identical tuple $(s_{x_v},p_{x_v},\ell_{x_v})$ over all such moments. The diameter $L$ of $\cP$ may not be identical over all such moments, so we use $L(i)$ to denote the value of $L$ for the \emph{last} such moment. We define $u(i) = L(i) /\ell(i)$.
Similarly,
we use $(s(\bot), p(\bot))$ to denote the identical value of $(s_{x_v},p_{x_v})$ whenever $t_{x_v} = \bot$.

Our goal is to show each of $p(0),p(1),p(2),p(\bot)$, whenever exists, is at least $1$.

For $i\in\{0,1,2\}$, whenever $p(i)$ exists, by \eqref{rule:p1} we have
\begin{equation}
\label{eq:p-product}
p(i)  = p(0)/\prod_{i'=0}^{i - 1}\big(2^{24} u(i')^{1 - 2/(d - i')}\big) \ge p(0)2^{-24i}/\left(\prod_{i'=0}^{i-1} u(i')\right)^{1 - 2/d}.
\end{equation}
By \eqref{rule:s} we have
\[
1 \le s(i) \le k\prod_{i' = 0}^{i - 1}\left(48\cdot 2^{i'}{d\choose i'}(\log_2 k)/u(i')\right).
\]
Therefore,
\[
\prod_{i'=0}^{i-1} u(i') \le k\prod_{i'=0}^{i-1}\left(48\cdot 2^{i'}{d\choose i'}\log_2 k\right).
\]
Combining this with \eqref{eq:p-initial} and \eqref{eq:p-product},
\begin{align}
p(0) & = 2^{54}k^{1 - 2/d}d^3(48\log_2 k)^3 \ge 65^2 > 1,\label{eq:p-initial-0}\\
p(1) & \ge 2^{30}d^3(48\log_2 k)^2 \ge 65^2 > 1,\nonumber\\
p(2) & \ge 2^5d^2(48\log_2 k) \ge 65^2 > 1.\nonumber
\end{align}
Finally, we show $p(\bot)\ge 1$ whenever $p(\bot)$ exists. If the transition of $t_{x_v}$ to $\bot$ happens at preprocessing, it is clear that $p(\bot)\ge \min\{p(0), p(1),p(2)\}/65^2 \ge 1$. Otherwise, by \eqref{rule:p1} and \eqref{rule:p2} we have
\begin{equation}
\label{eq:p-bot}
p(\bot) = p(0)/\left(4u(2)\prod_{i = 0}^{1}(2^{24} u(i)^{1 - 2/(d - i)})\right) = p(0)2^{-50}/(u(0)^{1 - 2/d}u(1)^{1 - 2/(d-1)}u(2)).
\end{equation}
By \eqref{rule:s}, we have
\[
1 \le s(\bot) \le k \prod_{i=0}^2\left(48 \cdot 2^i{d\choose i}(\log_2 k)/u(i)\right).
\]
Therefore,
\begin{equation}
\label{eq:s-bot}
u(0)u(1)u(2)\le  4k d^3 (48\log_2 k)^3.
\end{equation}
From the update rule \eqref{rule:ell}, we know
\[
u(2) = \frac{L(2)}{\ell(2)} \le \frac{L(1)}{\ell(2)} = u(1)^{1/(d - 1)}2^{-12} \le u(1)^{1/(d - 1)},
\]
which implies 
\[
u(2)^{2/d}\le u(1)^{2/(d(d - 1))} = u(1)^{2/(d - 1) - 2/d}.
\]
Therefore,
\begin{equation}
\label{eq:u012}
u(0)^{1 - 2/d}u(1)^{1 - 2/(d - 1)}u(2) \le u(0)^{1 - 2/d}u(1)^{1 - 2/d}u(2)^{1 - 2/d} \le 4k^{1 - 2/d}d^3(48\log_2k)^3,
\end{equation}
where the last inequality is by \eqref{eq:s-bot}. Plugging \eqref{eq:p-initial-0} and \eqref{eq:u012} into \eqref{eq:p-bot}, we get $p(\bot)\ge 1$, as desired.
\end{proof}

\begin{lemma}
\label{lm:induction}
In the process of running algorithm $\postprocess$, 
whenever $\decisiontree$ is called with input being subproblem $\cP = (X, Y, (\sigma_x)_{x\in X}, (\ell_x)_{x\in X}, (t_x)_{x\in X}, (c_x)_{x\in X}, (s_x)_{x\in X}, (p_x)_{x\in X})$, the output $((\delta_x)_{x\in X}, T)$ of $\decisiontree$ satisfies the following properties.
The decision tree $T$ has at most $|Y|$ leaves.
For every leaf $v$ of the decision tree $T$ and
every pair of points $x,x'\in X$ in the region defined by $v$, we have $\delta_x = \delta_x'$. Moreover,
$\sum_{x\in X}\|x - \delta_x\|_\infty^2 \le A(\cP)$.
\end{lemma}
\begin{proof}
The proof is essentially the same as the proof of \Cref{lm:2d-induction}. The only difference is that when $L = 0$, to prove $\sum_{x\in X}\|x - \delta_x\|_\infty^2 \le A(\cP)$, we need $p_x\ge 1$ from \Cref{lm:p}.
\end{proof}
\ifdefined\soda
Now we finish the proof of \Cref{thm:d>2}.
\begin{proof}
\else
\begin{proof}[Proof of \Cref{thm:d>2}]
\fi
The proof is essentially the same as the proof of \Cref{thm:2d} except that we use \Cref{lm:cutting} instead of \Cref{lm:2d-cutting}, and \Cref{lm:induction} instead of \Cref{lm:2d-induction}. In particular, the cost of the explainable clustering $\cC'$ is bounded as follows:
\begin{align*}
\cost(\cC')
= {} & \sum_{i=1}^n\|x_i - y_{\xi'(i)}\|_2^2 \le d \sum_{i=1}^n\|x_i - y_{\xi'(i)}\|_\infty^2 =  d\sum_{i=1}^n\|x_i - \delta_{x_i}\|_\infty^2 \le  dA(\widetilde\cP)\\
\le {} & O(k^{1 - 2/d}(\log k)^8(\log\log_2(2k))^3d^4)\cost(\cC).
\tag{by \eqref{eq:AP}}
\end{align*}
\end{proof}
\section{Lower Bound}
\label{sec:lb}
We prove a lower bound of $k^{1 - 2/d}/\polylog(k)$ on the competitive ratio for $d$-dimensional explainable $k$-means for all $k,d\ge 2$ (\Cref{thm:lb}).
Our proof is based on a construction by \cite{laber2021price} summarized in \Cref{lm:grid} below. For brevity, we do not repeat its proof here.
The construction allows us to show a competitive ratio lower bound depending on two other parameters $p$ and $b$ in \Cref{lm:lb-instance}. We then prove \Cref{thm:lb} by specifying the values of $p$ and $b$.
\begin{lemma}[\cite{laber2021price}]
\label{lm:grid}
Let $b,p$ be positive integers satisfying $b \ge 3$. There exists $b^p$ points in $p$ dimensions with the following properties:
\begin{enumerate}
\item \label{lm:grid-1} For every $j\in[p]$, the $j$-th coordinate of the $b^p$ points form a permutation of $\{0,\ldots,b^p - 1\}$.
\item \label{lm:grid-2} The $\ell_2$ distance between any two of the points is at least $b^{p-1}/2$.
\end{enumerate}
\end{lemma}
\begin{lemma}
\label{lm:lb-instance}
Given positive integers $k,d,p,b$ satisfying $p\le d, b\ge 3, b^p \le k$, there exists a set of points in $d$ dimensions for which any $k$-explainable clustering has competitive ratio $\Omega(b^{p - 2}/p)$ for the $k$-means cost.
\end{lemma}
\begin{proof}
When $p = 1$, it is trivial to show a competitive ratio lower bound of $1=\Omega(b^{p-2}/p)$, so we assume $p\ge 2$.

Let $Z$ denote the set of $b^p$ points $\bR^p$ from \Cref{lm:grid}. Every point $z\in Z$ creates a point $u(z)\in\bR^d$, where the first $p$ coordinates of $u(z)$ are equal to the $p$ coordinates of $z$, and the remaining $d - p$ coordinates of $u(z)$ are zeros. Let $Y_1=\{u(z):z\in Z\}$ be the resulting set of $b^p$ points in $\bR^d$. We construct a set $Y_2\subseteq \bR^d$ consisting of $k - b^p$ points that are sufficiently far away from each other and from the points in $Y_1$.

For every $y\in Y_1$, we create a set $X_y$ consisting of $2p$ points in $\bR^d$ as follows: for every $j\in[p]$, $X_y$ contains $2$ points $y_j^+, y_j^-\in \bR^d$ by adding $2/3$ and $-2/3$ to the $j$-th coordinate of $y$. For every $y\in Y_2$, we define $X_y$ to be the set containing only $y$ itself.

Given the $2pb^p + (k - b^p)$ points in $\bigcup_{y\in Y_1\cup Y_2} X_y$, there exists a $k$-clustering $\cC$ with $\cost(\cC) = O(pb^p)$ by choosing $Y_1\cup Y_2$ to be the centroids and assigning $x\in X_y$ to centroid $y$ for every $y\in Y_1\cup Y_2$. On the other hand, by \Cref{lm:grid} \Cref{lm:grid-1}, for any $k$-explainable clustering $\cC'$, there exist two points $x\in X_y,x'\in X_{y'}$ assigned to the same centroid in $\cC'$ with \emph{distinct} $y,y'\in Y_1\cup Y_2$ (otherwise, since $|Y_1\cup Y_2| = k$, for every $y\in Y_1$ there exists a centroid in $\cC'$ to which all points in $X_y$ are assigned, and by \Cref{lm:grid} \Cref{lm:grid-1} and the explainability of $\cC'$, all the points in $\bigcup_{y\in Y_1}X_y$ must be assigned to the same centroid in $\cC'$, a contradiction).
By \Cref{lm:grid} \Cref{lm:grid-2}, 
\[
\|x - x'\|_2 \ge \|y - y'\|_2 - 4/3 \ge b^{p-1}/2 - 4/3 \ge \Omega(b^{p-1}).
\]
Therefore, $\cost(\cC')\ge \|x - x'\|_2^2/2 \ge \Omega(b^{2p - 2})$. The competitive ratio is thus lower bounded by $\Omega(b^{p - 2}/p)$.
\end{proof}

\begin{theorem}
\label{thm:lb}
For every $k,d\ge 2$, there exists a set of points in $d$ dimensions for which any $k$-explainable clustering has competitive ratio $\Omega(k^{1 - 2/d}(\log k)^{-5/3}\log\log (2k))$ for the $k$-means cost.
\end{theorem}
\begin{proof}
We can assume $k\ge 3$ w.l.o.g.\ because 
it is trivial to show a competitive ratio lower bound of $1$.

We specify the integers $b$ and $p$ in \Cref{lm:lb-instance} to get concrete lower bounds for the competitive ratio.

When $k^{1/d} \ge 3d$, we choose $p = d$ and $b = \lfloor k^{1/p} \rfloor$. It is clear that $b \ge 3p$, so 
\[
b^p =  (b+1)^p / (1 + 1/b)^p \ge k/\exp(p/b) \ge \Omega(k).
\]
This implies that $b^{p - 2} = b^p/b^2 \ge \Omega(k^{1 - 2/d})$.
\Cref{lm:lb-instance} gives us a competitive ratio lower bound of
\[
\Omega(b^{p-2}/p)\ge \Omega(k^{1 - 2/d}/d) \ge \Omega(k^{1 - 2/d}(\log k)^{-1}\log\log k),
\]
where the last inequality is by \Cref{claim:log-loglog}.

When $(\log k)^{2/3}/\log\log k\le k^{1/d} < 3d$, we choose $p$ to be the maximum integer such that $k^{1/p} \ge 3p$, and choose $b = \lfloor k^{1/p}\rfloor$. Again we have $b \ge 3p$, so $b^p\ge \Omega(k)$. We also have $b \le k^{1/p} \le (k^{1/(p+1)})^{1 + 1/p} \le (p+1)^{1 + 1/p} \le O(p) = O(\log k/\log\log k)$, where the last inequality is by \Cref{claim:log-loglog}. This implies $b^{p - 2}/p = b^p/(b^2p) \ge \Omega(k(\log k)^{-3}(\log\log k)^{3})$.
The competitive ratio lower bound from \Cref{lm:lb-instance} is
\[
\Omega(b^{p - 2}/p)\ge 
\Omega(k(\log k)^{-3}(\log\log k)^{3}) \ge \Omega(k^{1 - 2/d}(\log k)^{-5/3}\log\log k).
\]

When $3 \le k^{1/d} < (\log k)^{2/3}/\log\log k$, we choose $b = \lceil k^{1/d}\rceil$ and $p = \lfloor \log_b k \rfloor$. 
We have $b \le 2k^{1/d}$ and thus 
$b^{p - 2} = b^{p+1}/b^3\ge k/b^3\ge k^{1 - 3/d}/8$. 
The competitive ratio lower bound from \Cref{lm:lb-instance} is
\[
\Omega(b^{p - 2}/p)\ge 
\Omega(k^{1 - 3/d}/p) \ge \Omega(k^{1 - 3/d}/d) \ge \Omega(k^{1-3/d}/\log k) \ge \Omega(k^{1 - 2/d}(\log k)^{-5/3}\log\log k).
\]

When $k^{1/d} < 3$, we choose $b = 3$ and $p = \lfloor \log_3 k\rfloor$. We have $b^{p-2} = b^{p+1}/b^3 \ge k/27$. 
The competitive ratio lower bound from \Cref{lm:lb-instance} is
\[
\Omega(b^{p-2}/p) \ge \Omega(k/p) \ge 
\Omega(k/\log k) \ge \Omega(k^{1 - 2/d}(\log k)^{-1}).\ifdefined\soda\else\qedhere\fi
\]
\end{proof}

\appendix
\ifdefined\soda\else
\setcounter{equation}{0}
\fi
\section{Helper lemmas and claims}
\label{sec:helper}
\begin{claim}
\label{claim:super-convexity}
Define $\alpha(z) = z^{1 + 2u}$ for $u\in (0,1/2)$. Suppose non-negative real numbers $z_1,z_2,z,\Delta$ satisfy $z_1 + z_2 \le z + u\Delta$, and $\Delta\le \min\{z_1,z_2\}$. Then $\alpha(z_1) + \alpha(z_2) \le \alpha(z)$.
\end{claim}
\begin{proof}
Assume w.l.o.g. $z_1\le z_2$.
By the monotonicity and convexity of $\alpha$, we have
\[
\alpha(z_1) + \alpha(z_2) \le \alpha(z_1) + \alpha(z + u\Delta - z_1) \le \alpha(\Delta) + \alpha(z + u\Delta - \Delta).
\]
We only need to prove that
\begin{equation}
\label{eq:Delta}
\Delta^{1 + 2u} + (z -(1 - u)\Delta)^{1 + 2u} \le z^{1 + 2u}.
\end{equation}
Note that $0 \le \Delta \le z/(2 - u)$ because $2\Delta \le z_1 + z_2 \le z + u\Delta$.
Since the left-hand-side is convex in $\Delta$, we just need to check \eqref{eq:Delta} when $\Delta = 0$ and $\Delta = z/(2-u)$. It reduces to checking $1 + 2u \ge \frac{\log 2}{\log (2-u)}$. This holds by a standard calculation using our assumption $u\in (0,1/2)$.
\end{proof}
\begin{claim}
\label{claim:monotone-relaxed}
If $0 \le u\le v \le 3M/4$, then $u\log (M/u) \le 2 v\log (M/v)$. 
\end{claim}
\begin{proof}
The function $\alpha(z) = z\log (M /z)$ is non-decreasing when $z \le M/e$, and non-increasing when $z \ge M/e$. We just need to check the claim for $u = M/e$ and $v = 3M/4$, which follows from a standard calculation.
\end{proof}
\begin{claim}
\label{claim:disjoint-intervals}
Let $I_1,\ldots,I_n\subseteq\bR$ be intervals. There exists $S\subseteq[n]$ such that $I_u\cap I_v = \emptyset$ for all distinct $u,v\in S$ and 
\[
|\bigcup_{u\in [n]}I_u|\le 3|\bigcup_{v\in S}I_v|.
\]
\end{claim}
This claim is a special form of the Vitali covering lemma. For completeness, we give a proof of it as follows.
\begin{proof}
Assume w.l.o.g.\ that $|I_1|\ge \cdots \ge |I_n|$. For $u = 1,\ldots,n$, construct $S_u\subseteq [u]$ inductively as follows. We set $S_1 = \{1\}$, and for every $u > 1$, we set $S_u = S_{u - 1}$ if there exists $v\in S_{u-1}$ such that $I_u \cap I_{v} \ne \emptyset$, and set $S_u = S_{u-1}\cup \{u\}$ otherwise. Define $S = S_n$. It is clear that $I_u\cap I_v = \emptyset$ for all distinct $u,v\in S$.

For every $u\in [n]\backslash S$, there exists $\alpha(u)\in S$ such that $\alpha(u)< u$ and $I_{\alpha(u)}\cap I_{u}\ne \emptyset$. We define $\alpha(u) = u$ when $u\in S$. 
For every $v\in S$ and $u\in \alpha^{-1}(v)\backslash\{v\}$, we have $I_{v}\cap I_u \ne\emptyset$ and $|I_u|\le |I_{v}|$. Therefore, $|\bigcup_{u\in \alpha^{-1}(v)}I_u| \le 3|I_{v}|$.
We have
\[
|\bigcup_{u\in[n]}I_u| = |\bigcup_{v\in S}\bigcup_{u\in\alpha^{-1}(v)}I_{u}| \le \sum_{v\in S}|\bigcup_{u\in\alpha^{-1}(v)}I_{u}| \le \sum_{v\in S} 3|I_{v}| = 3|\bigcup_{v\in S}I_v|.\ifdefined\soda\else\qedhere\fi
\]
\end{proof}
\begin{lemma}[\cite{MR1337358}]
\label{lm:mean-value}
Let $m,M$ be positive real numbers with $M \ge 2m$.
Suppose $U:(0,L)\rightarrow [m, M/2]$ is a non-decreasing function, and there is a finite set $I\subseteq (0,L)$ such that $U$ is differentiable over $(0,L)\backslash I$. Then, there exists $z\in (0,L)\backslash I$ such that $U'(z) \leq (1/L) U(z)\log \frac{M}{U(z)}\log\log_2 \frac{M}{m}$.
\end{lemma}
\begin{proof}
Consider the function $\alpha(z):=-\log\log_2 \frac{M}{U(z)}$. It is clear that $\alpha(z)$ is non-decreasing and differentiable over $(0,L)\backslash I$. The function has derivative 
\[
\alpha'(z) = \frac{U'(z)}{U(Z)\log(M/U(z))}.
\]

Suppose there does not exist such $z\in(0,L)\backslash I$ that makes the desired inequality hold. 
This implies that the derivative of $\alpha(z)$ exceeds $(1/L)\log\log_2 \frac{M}{m}$ whenever $z\in (0,L)\backslash I$. This contradicts with the fact that $\alpha(z)$ is bounded between $-\log\log_2(M/m)$ and $0$.
\end{proof}
\begin{lemma}[\cite{MR1337358}]
\label{lm:mean-value-2}
Let $m,M$ be positive real numbers with $M \ge 2m$.
Suppose $V:(0,L)\rightarrow [m, M-m]$ is a non-decreasing function, and there is a finite set $I\subseteq (0,L)$ such that $V$ is differentiable over $(0,L)\backslash I$. Then, there exists $z\in (0,L)\backslash I$ such that $V'(z) \leq (2/L)M'\log \frac{M}{M'}\log\log_2 \frac{M}{m}$, where $M' = \min\{V(z), M - V(z)\}$.
\end{lemma}
\begin{proof}
When $V(L/2) \le M/2$, the lemma follows from \Cref{lm:mean-value} by setting $U(z) = V(z/2)$. When $V(L/2)\ge M/2$, the lemma follows from \Cref{lm:mean-value} by setting $U(z) = M - V(L - z/2)$.
\end{proof}
\begin{claim}
\label{claim:log-loglog}
For real numbers $k \ge 3$ and $p > 0$, if $k^{1/p} \ge p$, then $p \le 2\log k/\log\log k$.
\end{claim}
\begin{proof}
Assume for the sake of contradiction that $p > 2\log k/\log\log k$, then $k^{1/p} < (\log k)^{1/2}$. Therefore,
\[
p/k^{1/p} > 2(\log k)^{1/2}/\log\log k = (\log k)^{1/2}/\log((\log k)^{1/2})\ge 1,
\]
a contradiction.
\end{proof}
\bibliographystyle{alpha}
\bibliography{ref}
\end{document}